\documentclass[letterpaper]{article} 
\usepackage{aaai25}  
\usepackage{times}  
\usepackage{helvet}  
\usepackage{courier}  
\usepackage[hyphens]{url}  
\usepackage{graphicx} 
\usepackage{natbib}  
\usepackage{caption} 
\usepackage{algorithm}
\usepackage{algorithmic}

\usepackage{newfloat}
\usepackage{listings}

\usepackage{enumitem}
\usepackage{placeins}
\usepackage{multirow}
\usepackage{amsmath,amsthm,amssymb,amsfonts}
\usepackage{svg}
\usepackage{subcaption}

\DeclareCaptionStyle{ruled}{labelfont=normalfont,labelsep=colon,strut=off} 
\floatstyle{ruled}
\newfloat{listing}{tb}{lst}{}
\floatname{listing}{Listing}
\title{HVAdam: A Full-Dimension Adaptive Optimizer}
\author{
    Yiheng Zhang\textsuperscript{\rm 1}\equalcontrib, Shaowu Wu\textsuperscript{\rm 1}\equalcontrib, Yuanzhuo Xu\textsuperscript{\rm 1}, Jiajun Wu\textsuperscript{\rm 2}, Shang Xu\textsuperscript{\rm 3}, STEVE DREW\textsuperscript{\rm 2}, Xiaoguang Niu\textsuperscript{\rm 1} \thanks{Corresponding Author}
}
\affiliations{
    \textsuperscript{\rm 1} School of Computer Science, Wuhan University \\
    \textsuperscript{\rm 2} Department of Electrical and Software Engineering, University of Calgary \\
    \textsuperscript{\rm 3} Department of Computer Science, University College London

    \{gjryhxt,wshaow,xyzxyz,xgniu\}@whu.edu.cn, \{jiajun.wu1, steve.drew\}@ucalgary.ca, shang.xu.24@ucl.ac.uk
}

\usepackage{bibentry}

\begin{document}

\maketitle

\begin{abstract}
Adaptive optimizers such as Adam and RMSProp have gained attraction in complex neural networks, including generative adversarial networks (GANs) and Transformers, thanks to their stable performance and fast convergence compared to non-adaptive optimizers. A frequently overlooked limitation of adaptive optimizers is that adjusting the learning rate of each dimension individually would ignore the knowledge of the whole loss landscape, resulting in slow updates of parameters, invalidating the learning rate adjustment strategy and eventually leading to widespread insufficient convergence of parameters. In this paper, we propose HVAdam, a novel optimizer that associates all dimensions of the parameters to find a new parameter update direction, leading to a refined parameter update strategy for an increased convergence rate.
We validated HVAdam in extensive experiments, showing its faster convergence, higher accuracy, and more stable performance on image classification, image generation, and natural language processing tasks. Particularly, HVAdam achieves a significant improvement on GANs compared with other state-of-the-art methods, especially in Wasserstein-GAN (WGAN) and its improved version with gradient penalty (WGAN-GP). Code is available at \url{https://github.com/ChihayaAnn/HVAdam}.
\end{abstract}

%

\section{Introduction}
\label{problem}

\begin{figure}
    \centering
    \includegraphics[width=\linewidth]{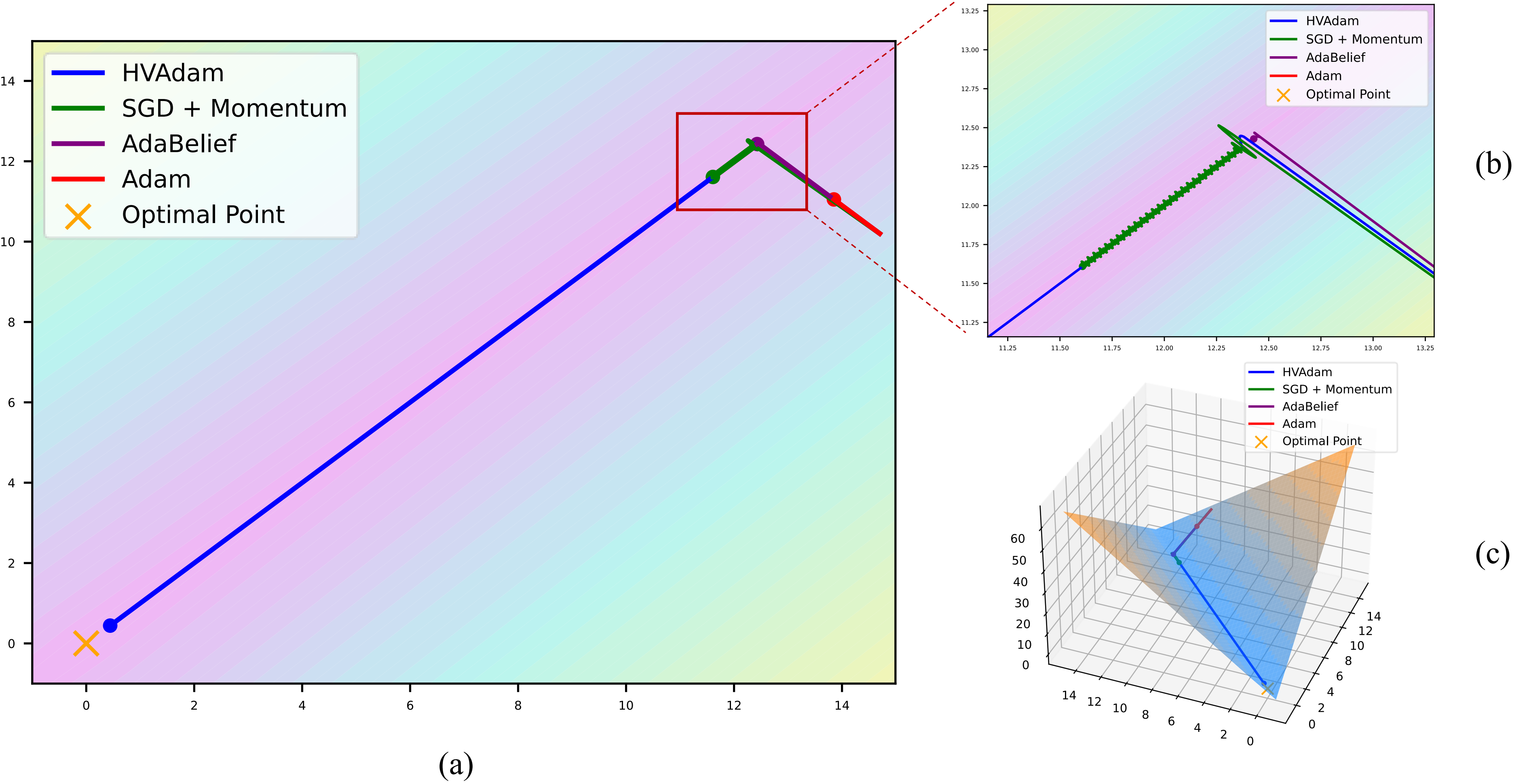}
    \caption{A typical example of the valley dilemma. (a) (c) depict the trajectories of {SGD}, {Adam}, {AdaBelief} and {HVAdam} in both 2D and 3D plots.  (b) is a close-up of the red box area in (a), showing the slow convergence and zigzagging behavior of Adam, AdaBelief, and SGD. However, HVAdam demonstrates rapid convergence along the hidden vector direction.}
    \label{fig:valley-dilemma}
\end{figure}
 \begin{figure*}[ht]
    \centering
    \includegraphics[width=\linewidth]{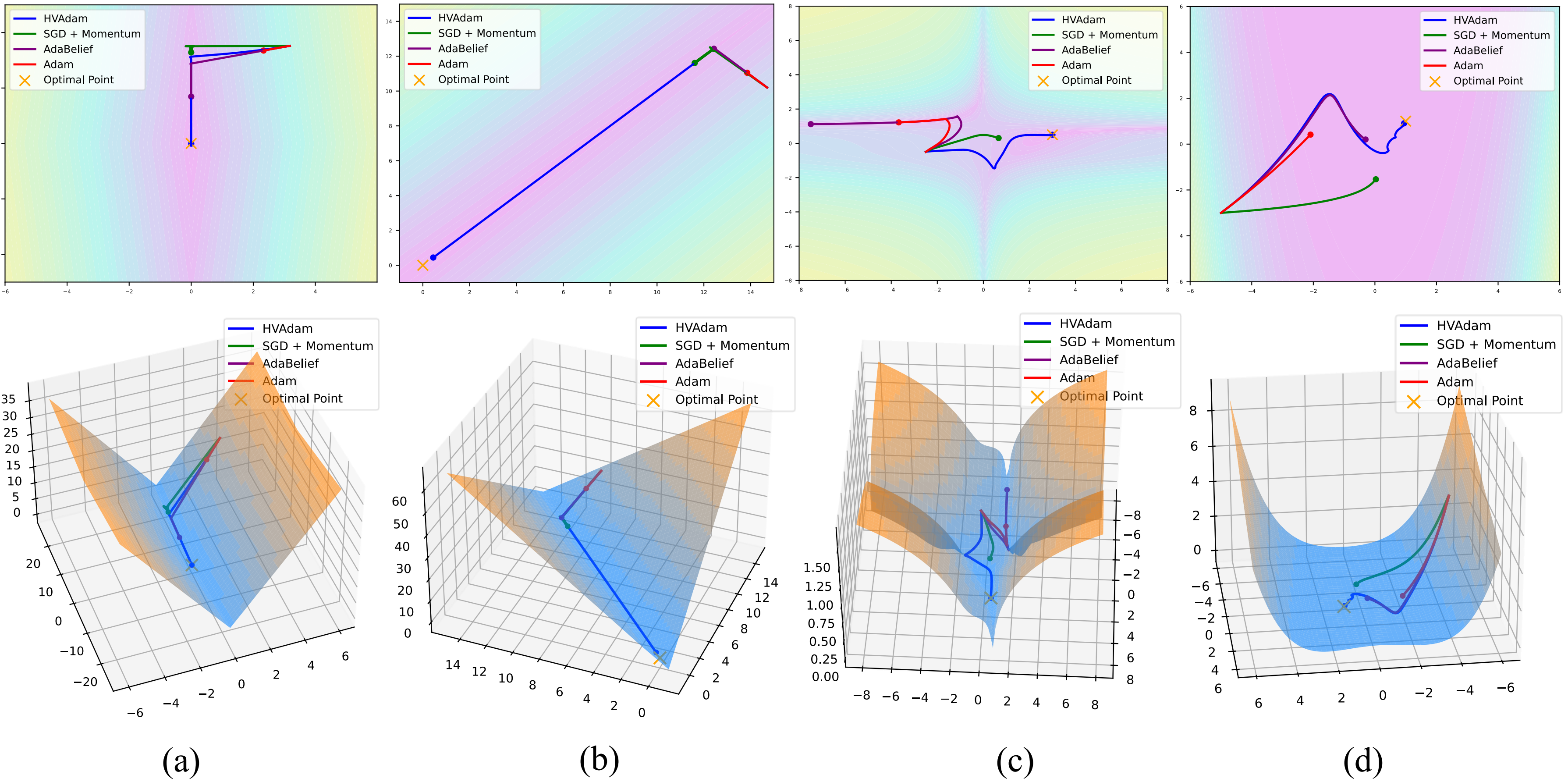}
    \caption{Trajectories of {SGD}, {Adam}, {AdaBelief} and {HVAdam}. The functions are $f_1$, $f_2$, $f_3$ and $f_4$ from (a) to (d). The functions are mentioned in the supplementary material's Sec.E. HVAdam reaches the optimal point (marked as orange cross in 2D and 3D plots) the fastest in all cases.}
    \label{fig:toy}
\end{figure*}

Optimizers play a crucial role in machine learning, efficiently minimizing loss and achieving generalization. There are two main types of state-of-the-art optimizers: Nonadaptive optimizers, for instance, stochastic gradient descent (SGD) \cite{robbins1951stochastic}, use a global learning rate for all parameters. In contrast, adaptive optimizers, such as RMSProp \cite{graves2013generating} and Adam \cite{kingma2014adam}, use the partial derivative of each parameter to adjust their learning rates separately. When applied correctly, adaptive optimizers can significantly accelerate the convergence of parameters, even in cases where partial derivatives are minimal. These adaptive optimizers are crucial for overcoming saddle points or navigating regions characterized by small partial derivatives, which can otherwise hinder the convergence process \cite{xie2022adaptive}.

 In deep learning, model parameters often change significantly during training, and these changes frequently reverse direction, especially when the model is getting close to its optimal performance (near the minimum of the loss function). Adam dynamically adjusts the learning rate according to the magnitude of the gradients, while SGD incorporates momentum to smooth the gradient and does not directly modify the learning rate. AdaBelief \cite{zhuang2020adabelief}, a variant of Adam, adjusts the learning rate based on the deviation of the gradient from the moving average.
However, these approaches may be problematic in complex landscapes, such as valley-like regions that are quite common within the loss functions of deep learning models \cite{zhuang2020adabelief}. In such regions, the true direction of optimization often has a small gradient, making it difficult for traditional optimizers to identify and efficiently optimize in these directions. Traditional optimizers tend to oscillate along valley sides, where gradients are larger, leading to slow and inefficient convergence. We refer to this phenomenon as the  \textit{Valley Dilemma}, which occurs when optimization techniques fail to converge effectively due to the challenging landscape as shown in Figure~\ref{fig:valley-dilemma}.
 
 To further demonstrate the \textit{Valley Dilemma}, we illustrate its effect on optimizers using two more common loss functions (e.g., (c) and (d) in Figure~\ref{fig:toy}). In valley-like regions, traditional optimizers face several key challenges: \textbf{1. Identifying Update Direction:} In valley-like regions, the true optimal direction (along the valley floor) often has a very small gradient, making it difficult for traditional optimizers to detect and follow this direction, as shown in Figure~\ref{fig:valley-dilemma}. \textbf{2. Changing Landscapes:} As training progresses, the direction of optimization can change, as shown in Figure~\ref{fig:toy}(b)(c), especially in complex and dynamic landscapes. \textbf{3. Adjusting Learning Rate:} Since this stable update direction incorporates trend information, the learning rate needs to be adjusted accordingly. To help convergence in other directions, we should design a more reasonable learning rate adjustment strategy based on the information of the trend for them.

To address the aforementioned challenges, we propose the hidden-vector-based adaptive gradient method (HVAdam), an adaptive optimizer that considers the full dimensions of the parameters. With the help of the \textit{hidden vector}, HVAdam can determine the stable update direction, which enables faster convergence along the descending trajectory of the loss function through an increased learning rate compared to the traditional zigzag approach. Additionally, HVAdam employs a restart strategy to adapt the change of the hidden vector, namely the changing direction of the landscapes.
Lastly, HVAdam incorporates information from all dimensions to adjust the learning rate for each dimension using a new preconditioning matrix based on the hidden vector. Contributions of our paper can be summarized as follows:

\begin{itemize}
    \item We propose HVAdam, which constructs a vector that approximates the invariant components within the gradients, namely the \textit{ hidden vector}, to more effectively guide parameter updates as a solution to the \textit{Valley Dilemma}.
    \item We demonstrate HVAdam's convergence properties under online convex and non-convex stochastic optimization, emphasizing its efficacy and robustness.
    \item We empirically evaluate the performance of HVAdam, demonstrating its significant improvements across image classification, NLP, and GANs tasks.
\end{itemize}

\section{Background and Motivation}
\paragraph{Notations} 
\begin{itemize}[topsep=0pt,parsep=0pt,partopsep=0pt]
    \item $f(\theta) \in \mathbb{R}, \theta \in \mathbb{R}^d$: $f(\theta)$ is the scalar-valued function to minimize, $\theta$ is the parameter in $\mathbb{R}^d$ to be optimal.
    \item  $\prod_{\mathcal{F},S}(y) = \mathrm{argmin}_{x \in \mathcal{F}} \vert \vert S^{1/2} (x-y) \vert \vert$:
    The projection of $y$ onto convex feasible set $\mathcal{F}$.
    \item $g_t\in \mathbb{R}^d$: The gradient at step $t$.
    \item $m_t\in \mathbb{R}^d$: The exponential moving average (EMA) of $g_t$.
    \item $v_t\in \mathbb{R}^d$: The hidden vector calculated by $v_{t-1}$ and $m_t$.
	\item $h_t\in \mathbb{R}$: The EMA of $\Vert m_t\Vert^2$.
    \item $p_t\in \mathbb{R}^d$: $p_t=(g_t - v_{t-1})^2$. $p_t$ is an intermediate variable.
	\item $\eta_t\in \mathbb{R}^d$: The factor measure the size of $p_t$.
	\item $a_t, s_t\in \mathbb{R}^d$: The EMA of $g_t^2$ and $\eta_t p_t$.
    \item $\alpha_1, \alpha_2 ,\gamma\in \mathbb{R}$: $\alpha_1$ is the unadjusted learning rate for $m_t$; $\alpha_2$ is the unadjusted learning rate for $v_t$ ; $\gamma$ is a constant to limit the value of $\eta_t$, which is usually set as $0$. These are hyperparameters.
	\item $\delta_t\in \mathbb{R}$: The factor used to adjust $\alpha_{2t}$.
	\item $\epsilon \in \mathbb{R}$: The hyperparameter $\epsilon$ is a small number, used for avoiding division by 0.
\item $lr(\delta_{t_2},\widehat{\delta_{t_2}})\in \mathbb{R}\times \mathbb{R} \rightarrow \mathbb{R}$: $\left\{
        \begin{array}{ll}
            10^{\delta_{t_2}\cdot6-3} ,& \text{if }\widehat{\delta_{t_2}} \ge 0.1 \\
            0,   & \text{otherwise}
        \end{array}
    \right.
$. The function can be set as other more suitable choices; we will not change it in the rest of the paper unless otherwise specified.
    \item $\beta_1, \beta_2\in \mathbb{R}$: The hyperparameter for EMA, $0\le \beta_1,\beta_2 < 1$, typically set as 0.9 and 0.999.
\end{itemize}

\subsection{Adaptive Moment Estimation (Adam)}

\begin{algorithm}[H]
\textbf{Input:} $\alpha_1$, $\beta_1$ , $\beta_2$, $\epsilon$\\
\textbf{Initialize} $\theta_0$, $m_0 \leftarrow 0$ , $v_0 \leftarrow 0$, $t \leftarrow 0$ 
\begin{algorithmic}[1]
\WHILE{$\theta_t$ not converged}
\STATE $t \leftarrow t + 1 $ 
\STATE $g_t \leftarrow \nabla_{\theta}f_t(\theta_{t-1})$ 
\STATE $m_t \leftarrow \beta_1 m_{t-1} + (1 - \beta_1) g_t$ 
\STATE $a_t \leftarrow \beta_2 a_{t-1} + (1 - \beta_2) g_t^2$ 
\STATE \textbf{Bias Correction} 
\STATE \hspace{3mm} $\widehat{m_t} \leftarrow \frac{m_t}{1-\beta_1^t}$, $\widehat{a_t} \leftarrow \frac{a_t}{1-\beta_2^t}$ 
\STATE  \textbf{Update} 
\STATE  \hspace{3mm} $\theta_t \leftarrow \prod_{\mathcal{F},\sqrt{\widehat{a_t}}} \Big( \theta_{t-1} - \frac{\alpha_1 \widehat{ m_t}} { \sqrt{ \widehat{ a_t}} +\epsilon } \Big)$
\ENDWHILE
\caption{Adam Optimizer}
\label{algo:adam}
\end{algorithmic}
\end{algorithm}

Adam integrates the merits of the adaptive gradient algorithm (AdaGrad) and root mean square propagation (RMSProp) by adaptively adjusting the learning rates based on estimates of the first and second moments of the gradients, making it exceptionally suitable for large-scale and complex machine learning problems. The algorithm is summarized in Algorithm ~\ref{algo:adam}, and all operations are element-wise.

In Algorithm.~\ref{algo:adam}, $\beta_1$ and $\beta_2$ are hyperparameters that represent the exponential decay rates of the moving averages of the past gradients and squared gradients, respectively. The learning rate is denoted by $\alpha$, and $\epsilon$ is a small constant added to the denominator to ensure numerical stability.

Adam adapts the learning rate for each parameter. Its straightforward implementation has made it a popular choice in the field of deep learning, particularly when working with large datasets or high-dimensional spaces.

\subsection{Problems and Motivation}
Traditional Adam and its variants do not effectively solve the valley dilemma. Adam adjusts the learning rate for each parameter. In regions with small partial derivatives or saddle points, the gradient $g_{t,i}$ at step $t$ is small, leading to a small $a_{t,i}$, since it is the EMA of $g_{t,i}^2$. As $\sqrt{a_{t,i}}$ is in the denominator, Adam takes a relatively large step in the $\theta_{t,i}$ direction due to this small denominator. This adjustment strategy is effective in escaping saddle points or regions with small partial derivatives because it enables larger updates in such scenarios, which is critical for maintaining the momentum of the optimizer \cite{xie2022adaptive}. However, for the valley dilemma case, all related $g_{t, i}$ can be large, so the corresponding $a_{t, i}$ will also be large. This results in small update step sizes for all parameters, including in directions where acceleration is needed for effective updating. Therefore, this direction is ``hidden'' from Adam. Adam adjusts the learning rates for each parameter separately, which allows it to perform well when dealing with problems where parameter updates are primarily aligned with the coordinate axes. However, this approach is less effective for problems that require significant updates in non-axis-aligned directions. This limitation highlights that Adam’s learning rate adjustment strategy is most effective for problems that are “close to axis-aligned”, as discussed in \cite{balles2018dissecting}. Most other first-order adaptive optimizers face similar challenges. Therefore, in the context of the non-axis-aligned valley dilemma, the critical challenge is to identify this ``hidden'' direction, which we refer to as the \textit{hidden vector}. 

We provide specific examples to clarify our intuitive explanations regarding the valley dilemma and hidden vector. As shown in Figure~\ref{fig:toy}, panels (a) and (b) depict two typical valley functions. The hidden vector, corresponding to the gradient at the bottom of the valley, represents the intersection line of two planes. The HVadam converges most rapidly towards the direction of the hidden vector and achieves the fastest convergence. Figure~\ref{fig:toy}(a) illustrates a typical coordinate-aligned valley problem, where AdaBelief also converges quickly. However, for the non-axis-aligned valley problem represented in Figure~\ref{fig:toy}(b), Adam’s convergence speed is relatively slow. We have further generalized the concept of the hidden vector to functions where the hidden vector frequently changes using a restart strategy. Furthermore, we have verified the convergence performance of HVAdam in more general cases, as shown in Figure~\ref{fig:toy}(c)(d), where HVAdam continues to demonstrate good convergence. These examples provide insight into the local behavior of optimizers in deep learning. Their behavior reflects common patterns seen in deep networks, such as ReLU activation, neuron connections, cross-entropy loss, and smooth activations \cite{zhuang2020adabelief}. Further details on the analysis of these examples are available in Sec.E of the supplementary material.

\section{HVAdam}
\begin{figure*}[ht]
\centering
\noindent
\scalebox{1.0}{
\begin{tabular}{c|ccccc}
\hline
    Step & 1  & 2 & 3  & 4 & 5  \\ \hline
    $g_x$   & 5  & -3 & 5  & -3 & 5  \\ \cline{2-6} 
    $g_y$  & -3 & 5 & -3 & 5 & -3 \\ \hline
    $\widehat{m_x}$   & 5  & 0.7895 & 2.3432  & 0.7895 & 1.8177  \\ \cline{2-6} 
                           $\widehat{m_y}$ & -3  & -1.2105 & -0.3432  & 1.2105 & 0.1823  \\ \hline
    $v_x$  & 5  & 1 & 1  & 1 & 1  \\ \cline{2-6} 
                           $v_y$   & -3  & 1 & 1  & 1 & 1  \\ \hline
    $v^*_x$  & 1  & 1 & 1  & 1 & 1  \\ \cline{2-6} 
                           $v^*_y$   & 1  & 1 & 1  & 1 & 1  \\ \hline
\end{tabular}
}
\begin{tabular}{@{}c@{}} 
\includegraphics[width=0.25\linewidth]{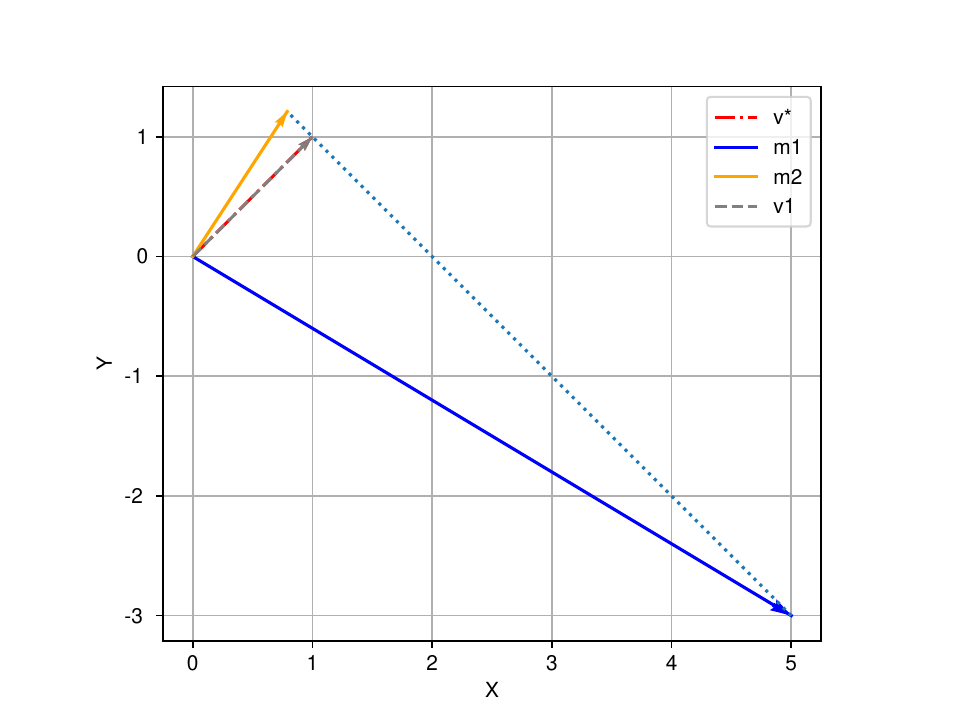}
\end{tabular}

\captionof{figure}{ Consider $f(x,y)=4\vert x-y \vert + \vert x+y \vert$. \textit{Left: } 
Optimization process for the example function. Our algorithm uses only two steps to get the hidden vector $v^*$. \textit{Right: } the figure shows how we make $v_t$ approximate $v^*$ for the function.}
\label{fig:hd}
\end{figure*}

\begin{algorithm}
\textbf{Input:} $\alpha_1$, $\beta_1$ , $\beta_2$, $\epsilon$, $\gamma$\\
\textbf{Initialize} $\theta_0$, $\alpha_1 \leftarrow \alpha_2$, $m_0 \leftarrow 0$ , $s_0 \leftarrow 0$, $v_0 \leftarrow 0$, $t \leftarrow 0$, $t_2 \leftarrow -1$, $\delta_0 \leftarrow 0$ 
\begin{algorithmic}[1]
\WHILE{$\theta_t$ not converged}
\STATE $t \leftarrow t + 1 $
\STATE $t_2 \leftarrow t_2 + 1 $
\STATE $g_t \leftarrow \nabla_{\theta}f_t(\theta_{t-1})$
\STATE $m_t \leftarrow \beta_1 m_{t-1} + (1 - \beta_1) g_t$
\STATE {$p_t \leftarrow {(g_t - v_{t-1})^2}$}
\STATE {$\eta_t \leftarrow \frac{p_t}{(g_t - m_t)^2 + \gamma p_t + \epsilon}$}
\STATE {$s_{t} \leftarrow \beta_2 s_{t-1} + (1 - \beta_2) {\eta_t p_t}+\epsilon$}
\STATE \textbf{Bias Correction}
\STATE \hspace{3mm} $\widehat{m_t} \leftarrow \frac{m_t} {1-\beta_1^t}$,  $\widehat{s_t} \leftarrow \frac{s_t} {1-\beta_2^t}$
\IF{$t_2$ \textbf{not} $0$}
\IF{$v_{t-1} = \widehat{m_t}$}
\STATE {$k_t=0$}
\ELSE
\STATE {$k_{t} \leftarrow \frac{\langle v_{t_2-1}-\widehat{m_t} , v_{t_2-1} \rangle}{\Vert v_{t_2-1}-\widehat{m_t} \Vert^2}$} 
\ENDIF
\STATE {$v_{t_2} \leftarrow {k_t \widehat{m_t}+(1-k_{t})v_{t_2-1}}$}
\STATE {$\delta_{t_2} \leftarrow \beta_2\delta_{t_2-1} + (1-\beta_2)\text{cos}\langle v_{t_2},\widehat{m_t}\rangle $}
\STATE {$\widehat{\delta_{t_2}} \leftarrow \frac{\delta_{t_2}}{1-\beta_2^{t_2}}$}
\STATE {$b_{t} \leftarrow lr(\delta_{t_2},\widehat{\delta_{t_2}})$}
\IF{$b_t = 0$}
\STATE{$t_2=-1$}
\ENDIF
\ELSE
\STATE $v_{0} \leftarrow \widehat{m_t}$, $\delta_0 \leftarrow 0$
\ENDIF
\STATE \textbf{Update} \\
\STATE \hspace{3mm} $\theta_t \leftarrow \prod_{\mathcal{F},\sqrt{ \widehat{s_t}}} \Big( \theta_{t-1} -  \frac{\alpha_1 \widehat{m_t}} { \sqrt{{\widehat{s_t}}} + {\epsilon}} - \alpha_2 b_t v_t \Big)$
\ENDWHILE
\caption{HVAdam Optimizer}
\label{algo:HVAdam}
\end{algorithmic}
\end{algorithm}
We propose HVAdam, a first-order, full-dimension optimizer designed to address the non-axis-aligned valley dilemma. It not only solves the valley dilemma but also proves effective for general deep learning optimization. Specifically, we obtain a hidden vector using gradient projection, which represents the stable gradient trend of the loss function. By employing the restart strategy, we extend this approach to situations where the hidden vector may change over time. Finally, we enhance HVAdam’s effectiveness through a hidden-vector-based preconditioning matrix adjustment strategy, where the hidden vector is used to adjust the learning rate. The HVAdam algorithm is summarized in Algorithm ~\ref{algo:HVAdam}. All operations are element-wise, except for $\Vert \cdot \Vert$ and $\langle \cdot , \cdot \rangle$. The proof of the optimizer's convergence is shown in the supplementary material's Sec.C and Sec.D.

\paragraph{Hidden Vector}

\begin{align}
\label{v^t's update}
&k_t := \left\{
        \begin{array}{ll}
             \frac{\langle v_{t-1}-\widehat{m_t} , v_{t-1} \rangle}{\Vert v_{t-1}-\widehat{m_t} \Vert^2} ,& \text{if }v_{t-1}\ne \widehat{m_t} \\
            0,   & \text{otherwise}
        \end{array}
    \right. ,\\ \label{v^t's update2}
&v_t := {k_t \widehat{m_t}+(1-k_t)v_{t-1}} ,
\end{align}

To obtain the hidden vector of the loss function, we analyze the relationship between the hidden vector of the bivariate function $f(x,y)=4\vert x-y \vert + \vert x+y \vert$ in its valley region and its corresponding vectors $v_t$ and $\widehat{m_t}$. The process is illustrated in Figure \ref{fig:hd}. We observe that the height of the triangle formed by the edges $v_t$ and $\widehat{m_t}$ can be used to update $v_t$, leading to its convergence to the hidden vector $v^*$. The update process can be formulated as Eq. (\ref{v^t's update}) and Eq. (\ref{v^t's update2}).  We extend the algorithm to higher dimensions and prove its convergence in Sec.B of the Supplementary Material. The update process of the hidden vector is detailed in line 12 \~{} 17 of Algorithm~\ref{algo:HVAdam}.

\paragraph{Restart Strategy}
Now we can calculate $v_t$ through Eq.~(\ref{v^t's update}) and Eq.~(\ref{v^t's update2}), however, $v_t$ is a value that gradually converges to $v^*$ over time. Therefore, it is necessary to measure the current convergence rate of $v_t$. Then, we update the parameters in the direction of $v_t$ according to this convergence rate.
Considering that the moving average of $m_t$ can reflect the region trend of the loss function, we use the cosine similarity between $v_t$ and $\widehat{m_t}$ to measure the convergence rate of $v_t$. Although $\widehat{m_t}$ can roughly reflect the current region trend, it is unstable and therefore cannot replace $v_t$. To make the results more stable, we use the moving average of cosine similarity as the index, which is the 18th line of Algorithm~\ref{algo:HVAdam}, that is
\begin{align}
\label{eq:18}
\delta_{t_2} := \beta_2\delta_{t_2-1} + (1-\beta_2)\text{cos}\langle v_{t_2},\widehat{m_t}\rangle,
\end{align}

Furthermore, $\widehat{m_t}$ is adaptive to any changes between different local trends, while $v_t$ cannot automatically adjust itself. 
Whenever the difference between $\widehat{m_t}$ and $v_t$ becomes too large, it indicates that the trend of the region is changing, which requires the reinitiation of the calculations $v_t$ at time $t$.
Considering that the initialized values of $v_t$ and $\widehat{m_t}$ do not represent an accurate estimate of the current trend, we introduce the unbiased estimate of $\delta_{t_2}$ as a criterion to determine whether a restart is needed,

\begin{align}
\label{eq:19}
\widehat{\delta_{t_2}} := \frac{\delta_{t_2}}{1-\beta_2^{t_2}}.
\end{align}
A larger $\delta_{t_2}$ means that the direction of $v_t$ is closer to the direction of $m_t$, and vice versa. Therefore, a larger step in the direction of $v_t$ can be taken, while a small $\delta_t$ indicates less confidence. When $\delta_t$ is extremely small, it implies that the hidden vector $v^*$  has changed, so we will restart the calculation of $v_t$.
 
Finally, we obtain the step size $b_t$ by ${\delta_{t_2}}$ which represents the convergence rate of $v_t$.
$lr(\cdot)$ should be an increasing function capable of covering a wide range of magnitudes. After some experiments, we empirically select Eq.~(\ref{eq:20}), 
following the same process as in \cite{luo2019adaptive}. Whether to restart, we use $\widehat{\delta_{t_2}}$ as the criterion. When $\widehat{\delta_{t_2}}$is less than the threshold of 0.1, it indicates a significant deviation between $v_t$ and $\widehat{m_t}$. In such cases, we recompute $v_t$ starting from the current step t. The value of 0.1 here is selected empirically. 
\begin{align}
\label{eq:20}
b_t = lr(\delta_{t_2},\widehat{\delta_{t_2}}) := \left\{
        \begin{array}{ll}
            10^{\delta_{t_2}\cdot6-3} ,& \text{if }\widehat{\delta_{t_2}} \ge 0.1 \\
            0,   & \text{otherwise}
        \end{array}
    \right.
\end{align}
The pseudocode for the restart strategy corresponds to lines 18 to 23 of Algorithm~\ref{algo:HVAdam}, where $t_2$ is used to indicate whether a restart is needed.

\paragraph{Hidden-vector-based preconditioning matrix adjustment strategy}
\label{subsec:direction}
A crucial element of adaptive optimizers is the preconditioning matrix, which improves the information about the gradient and controls the step size in each direction of the gradient \cite{yue2023agd}. In order to adjust the learning rate based on the stable trend information obtained from the hidden vector $v_t$, we measure the difference between the gradients at the current position and the trend in the current region,
\begin{align}
\label{eq:6}
p_t := {(g_t - v_{t-1})^2}.
\end{align}
This difference indicates the magnitude of the noise in each direction of the coordinate axis. The magnitude of noise determines the step size of parameter updates in this direction; large noise corresponds to small steps, while small noise corresponds to large steps. 
The $p_t$ only represents the absolute magnitude of the noise. For different orders of magnitude of $g_t$, it is necessary to incorporate the relative magnitude of $p_t$ on each dimension into the design of the learning rate adjustment strategy. Therefore, we introduce Eq.~(\ref{eq:7})  to extract the relative magnitude factor of the noise in each dimension. 
\begin{align}
\label{eq:7}
\eta_t := \frac{p_t}{(g_t - m_t)^2 + \gamma p_t + \epsilon}
\end{align}

We obtain the relative magnitude of the difference between $g_t$ and $v_t$ by comparing it to the difference between $g_t$ and $m_t$. And in order to reduce the impact of abnormal data that make the denominator too small, we add $\gamma p_t$. The $\epsilon$ is used to avoid a zero denominator.

Finally, the relative magnitude of noise $\eta_t p_t$ constitute the adjustment factor of the learning rate,  formalized as
\begin{align}
\label{eq:8}
s_{t} := \beta_2 s_{t-1} + (1 - \beta_2) {\eta_t p_t}+\epsilon .
\end{align}
The hidden vector-based preconditioning matrix adjustment strategy is in the line 6\~{}8 of Algorithm~\ref{algo:HVAdam}.
If $p_t$ is large, it means that there is a significant difference between the projection of the gradient and $v_t$ on the parameters. In this case, as the denominator, $\sqrt{s_t}$ is the EMA of $\eta_tp_t$, making $\alpha_{1}$ small which makes updating more ``cautious''. If $p_t$ is small, this means that we should accelerate the update so the small $\sqrt{s_t}$ makes $\alpha_1$ large. For the denominators of Adam and AdaBelief, their value ranges are narrow. Adam's $\sqrt{a_t}$ is in $(0,\mathrm{max}\vert g\vert)$ and AdaBelief's is in $ (0,\mathrm{max}\vert 2g_t\vert)$. And the range of the denominator determines the adjustment range of $\alpha_1$. For HVAdam, we multiply $p_t$ by $\eta_t$. We use $(g_t-m_t)^2$ to measure $p_t$. If the latter is larger than the former, $\eta_t$ is decreased. Otherwise, $\eta_t$ is increased. The formula shows that the value range of $\eta_t$ is $(0,\frac{1}{\gamma})$, which can make the range wider so that an optimal learning rate can be reached.

\section{Validation on Tasks in Deep Learning}
\begin{figure*}[ht]
\begin{subfigure}[b]{0.33\textwidth}
\includegraphics[width=\linewidth]{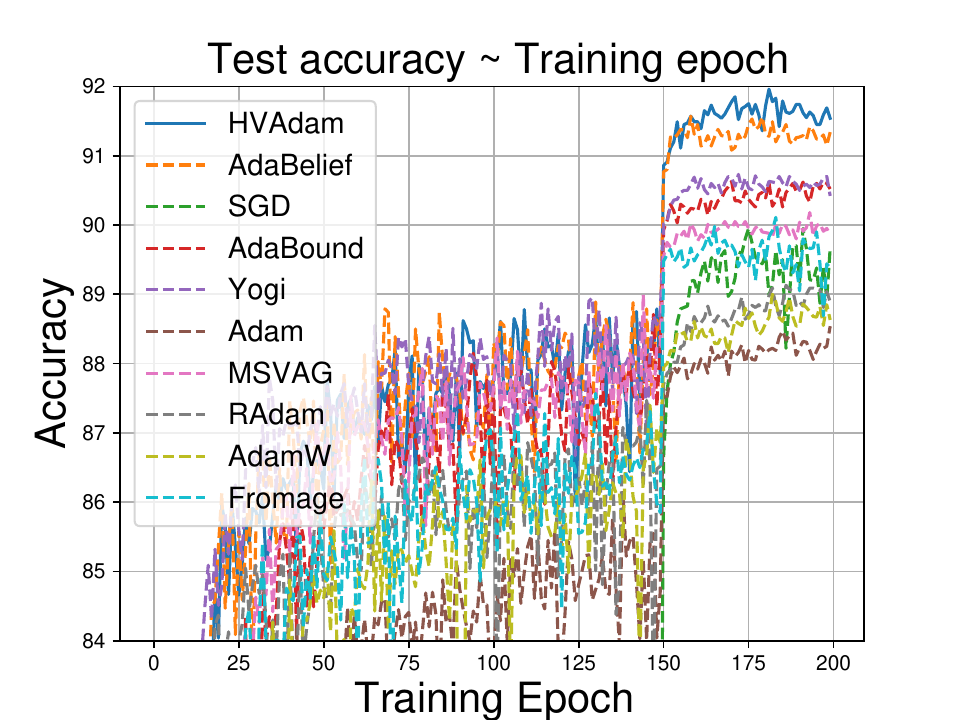}
\caption{\small{
VGG11 on CIFAR-10
}}
\end{subfigure}
\begin{subfigure}[b]{0.33\textwidth}
\includegraphics[width=\linewidth]{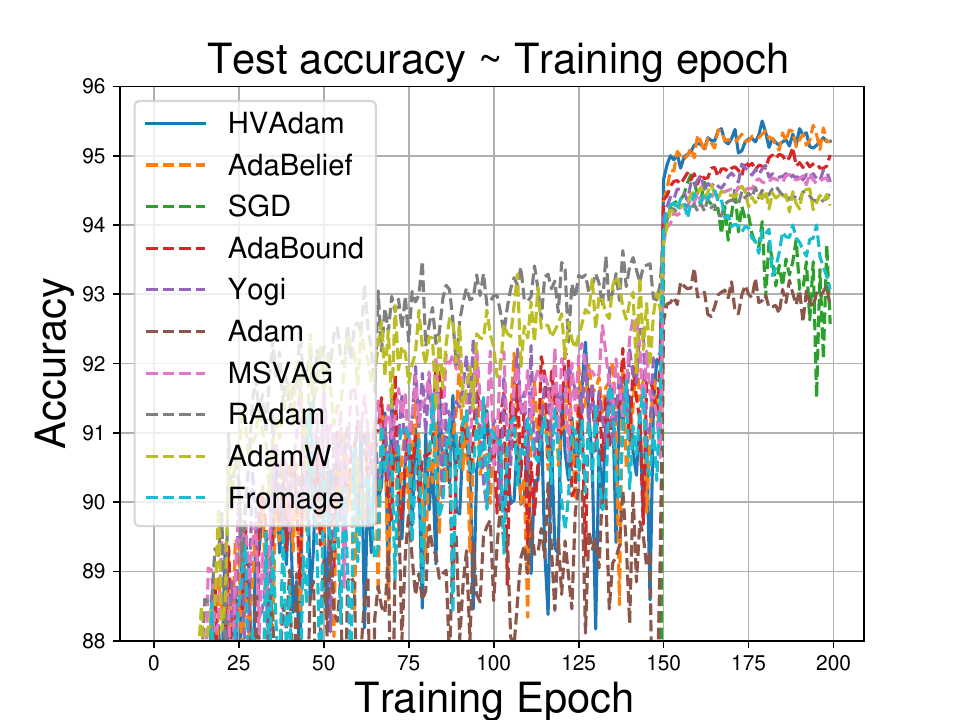}
\caption{\small{
ResNet34 on CIFAR-10
}}
\end{subfigure}
\begin{subfigure}[b]{0.33\textwidth}
\includegraphics[width=\linewidth]{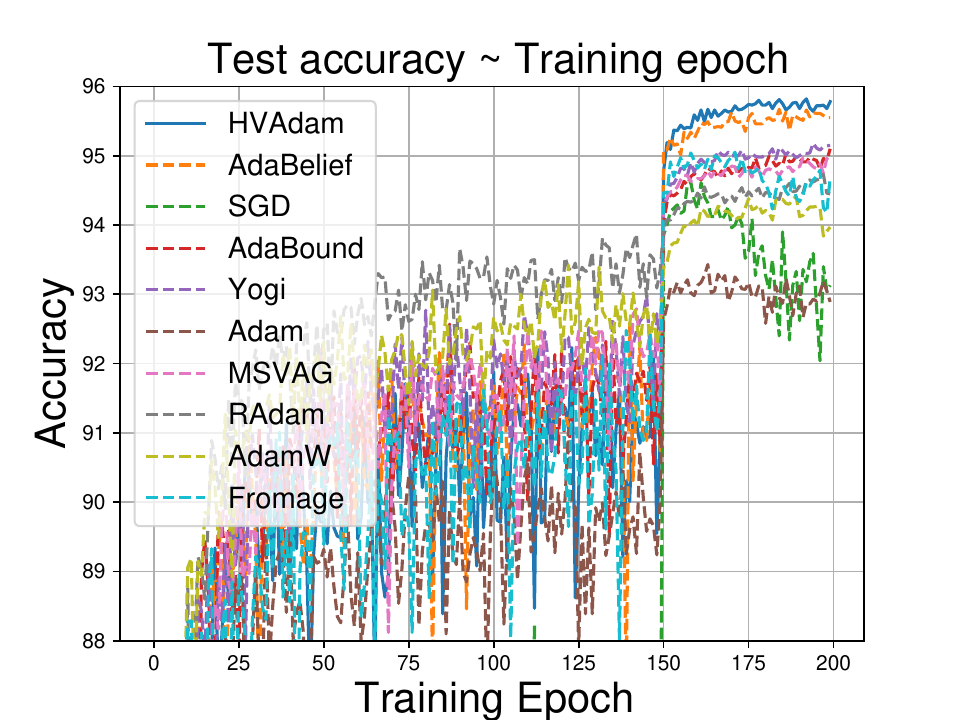}
\caption{\small{
DenseNet121 on CIFAR-10
}}
\end{subfigure}
\\
\begin{subfigure}[b]{0.33\textwidth}
\includegraphics[width=\linewidth]{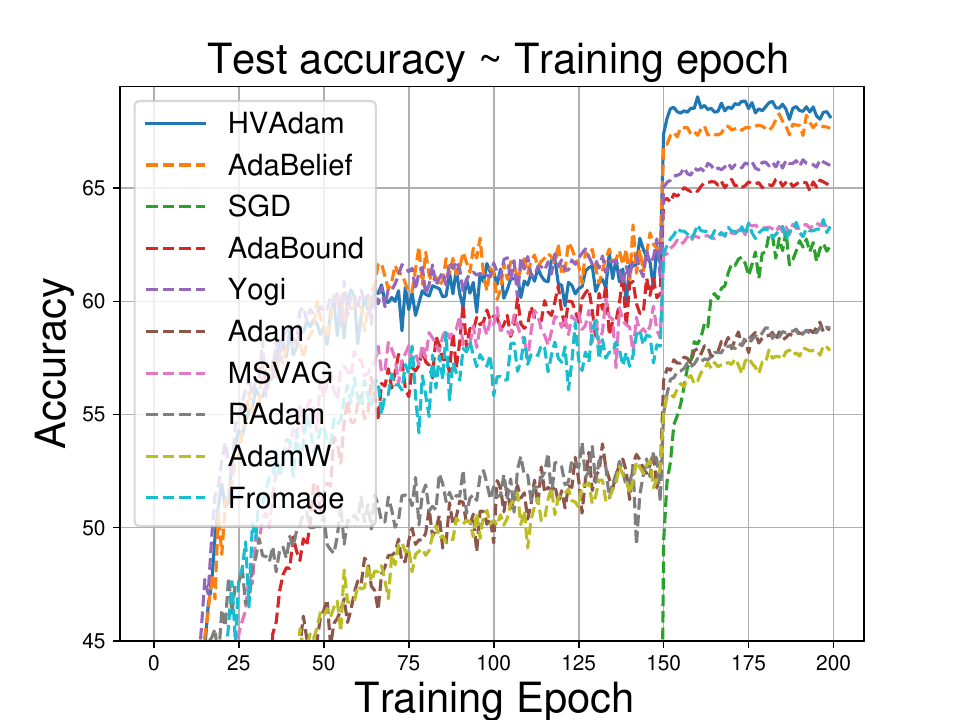}
\caption{\small{
VGG11 on CIFAR-100
}}
\end{subfigure}
\begin{subfigure}[b]{0.33\textwidth}
\includegraphics[width=\linewidth]{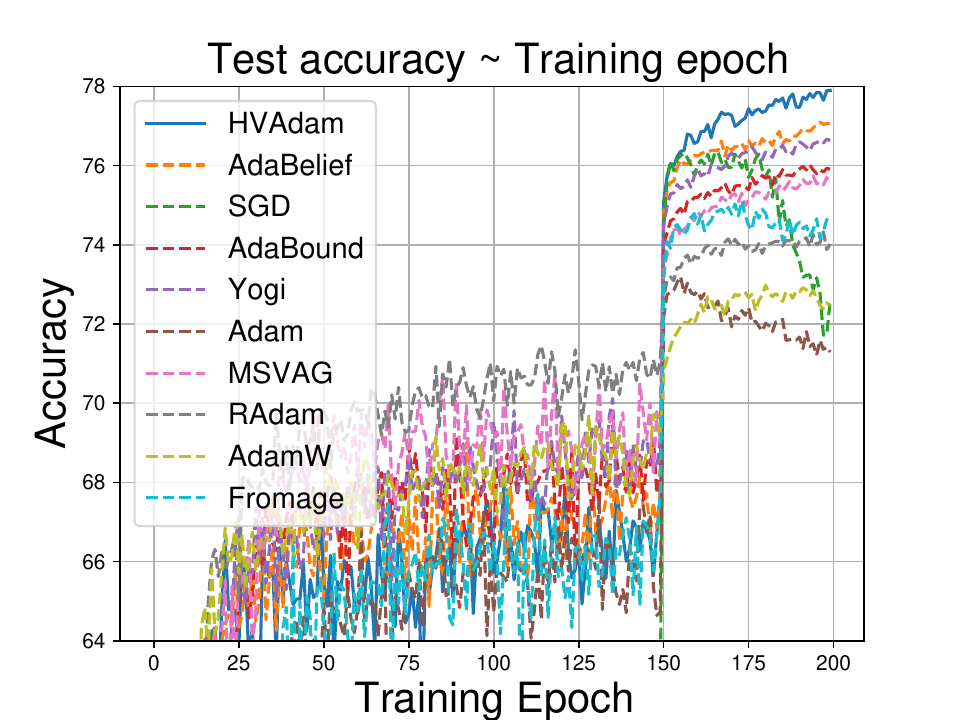}
\caption{\small{
ResNet34 on CIFAR-100
}}
\end{subfigure}
\begin{subfigure}[b]{0.33\textwidth}
\includegraphics[width=\linewidth]{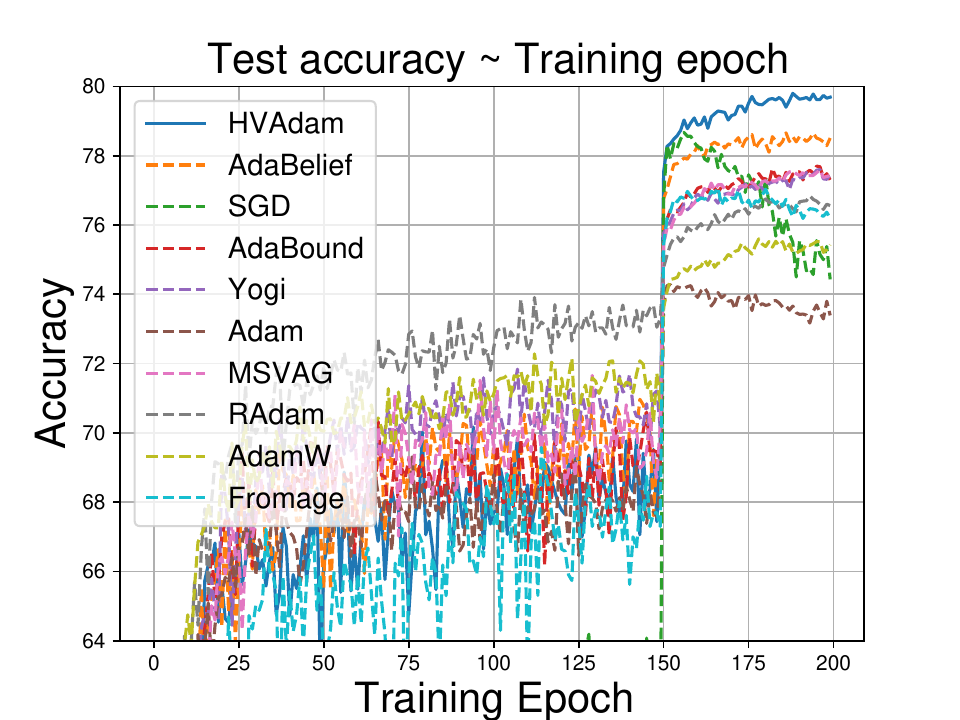}
\caption{\small{
DenseNet121 on CIFAR-100
}}
\end{subfigure}
\caption{Test accuracies with three models using different optimizers on CIFAR-10 and CIFAR-100.}
\label{fig:cifartext}
\end{figure*}

We compare HVAdam with 15 baseline optimizers in various experiments, including SGD \cite{sutskever2013importance}, Adam, MSVAG \cite{balles2018dissecting}, AdamW \cite{loshchilov2017decoupled}, Yogi \cite{zaheer2018adaptive}, AdaBound, RAdam \cite{liu2019variance}, Fromage \cite{bernstein2020distance}, RMSProp, SWA\cite{PavelIzmailov2018AveragingWL}, Lookahead\cite{zhang2019lookahead}, AdaBelief, Adai \cite{xie2022adaptive} and Lookaournd\cite{zhang2023Lookaround}. The experiments include: (a) image classification in the CIFAR-10 dataset and CIFAR-100 dataset \cite{krizhevsky2009learning} with VGG \cite{simonyan2014very}, ResNet \cite{he2016deep} and DenseNet \cite{huang2017densely}; (b) natural language processing tasks with LSTM \cite{ma2015long} on Penn TreeBank dataset \cite{marcus1993building} and Transformer on IWSLT14 dataset; (c) WGAN \cite{arjovsky2017wasserstein}, WGAN-GP \cite{gulrajani2017improved} and Spectral-Norm GAN (SN-GAN) \cite{miyato2018spectral} on CIFAR-10. The hyperparameter settings and searching are shown in supplementary material.

\subsection{Experiments for Image Classification}
\paragraph{CNNs on image classification}
The experiments are conducted on CIFAR-10 with VGG11, ResNet34, and DenseNet121 on CIFAR-10 using the \textit{ official implementation} of AdaBelief. 
The accuracy of other optimizers is obtained from \cite{buvanesh2021re}. As Figure ~\ref{fig:cifartext} shows, HVAdam achieves a fast convergence comparable to other adaptive methods. When the training accuracy of several optimizers is close to $100\%$, HVAdam achieves the highest test accuracy and outperforms other optimizers. The results show that HVAdam has both fast convergence and high generalization performance. 

\begin{table*}[t]
\centering
\setlength{\tabcolsep}{1mm}
\small
\begin{tabular}{ccccccc}
\hline
    SGDM\dag &Adam\dag &Adai\dag & SWA\ddag & Lookahead\ddag &Lookaround\ddag &HVAdam           \\ \hline
    76.49 &72.87 &76.80 & 76.78 &76.52 &\textbf{77.32} & \underline{77.22}           \\ \hline
\end{tabular}
\caption{Top-1 accuracy of ResNet50 on ImageNet. \dag is reported in \cite{xie2022adaptive}, \ddag is reported in \cite{zhang2023Lookaround}.}
\label{table:imagenet}
\end{table*}

We then train a ResNet50 on ImageNet \cite{deng2009imagenet} and report the accuracy on the validation set in Table \ref{table:imagenet}. The experiment is conducted using the \textit{ official implementation} of Lookaround.  For other optimizers, we report the best result in the literature. Due to the heavy computational burden, we were unable to perform an extensive hyperparameter search. But our optimizer still outperforms other adaptive methods and achieves comparable accuracy Lookaround (77.22 vs 77.32), which closes the generalization gap between adaptive methods and non-adaptive methods. Experiments validate the fast convergence and great generalization performance of HVAdam.

\paragraph{Visual Transformer on image classification} Besides validating with the classical model CNNs, we also validate the performance of HVAdam with Visual Transformer (ViT)\cite{dosovitskiy2020vit}. 
As Table \ref{table:Vit} shows, HVAdam outperforms all other optimizers, which means that HVAdam performs well on both classical and advanced models.

\begin{table}[h]
\centering
\setlength{\tabcolsep}{1mm}
\small
\begin{tabular}{c|ccccc}
\hline
             &Adam & SWA & Lookahead & Lookaround &HVAdam           \\ \hline
     CIFAR10 &98.34 & 98.47 & 98.51 & {98.71} &\textbf{99.00}           \\ \hline
     CIFAR100  &91.55 & 91.32 & 91.76 & {92.21} &\textbf{92.38}    \\ \hline
\end{tabular}
\caption{The test set accuracy under optimizers using ViT-B/16. The results of other optimizers are taken from \cite{zhang2023Lookaround}.}
\label{table:Vit}
\end{table}

\subsection{Experiments for Natural Language Processing}
\paragraph{LSTMs on language modeling} We experiment with 1-layer, 2-layer, and 3-layer LSTM models on the Penn TreeBank dataset. The results of the experiments are shown in Figure ~\ref{fig:LSTMtext}. Except for HVAdam, the score data for the other optimizers is provided by \cite{buvanesh2021re}. 
For all LSTM models, AdaBelief achieves the lowest perplexity or the best performance. The experiments resonate with both the fast convergence and the excellent accuracy of HVAdam. 

\begin{figure*}[t]
    \begin{subfigure}[b]{0.32\textwidth}
    \includegraphics[width=\linewidth]{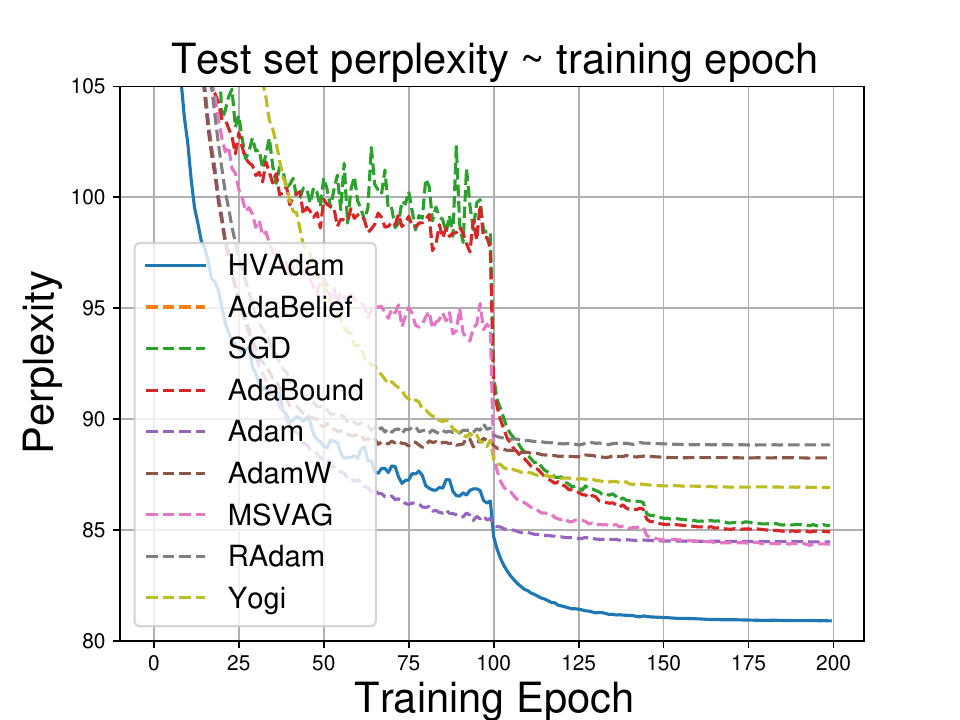}
    \end{subfigure}
    \begin{subfigure}[b]{0.32\textwidth}
    \includegraphics[width=\linewidth]{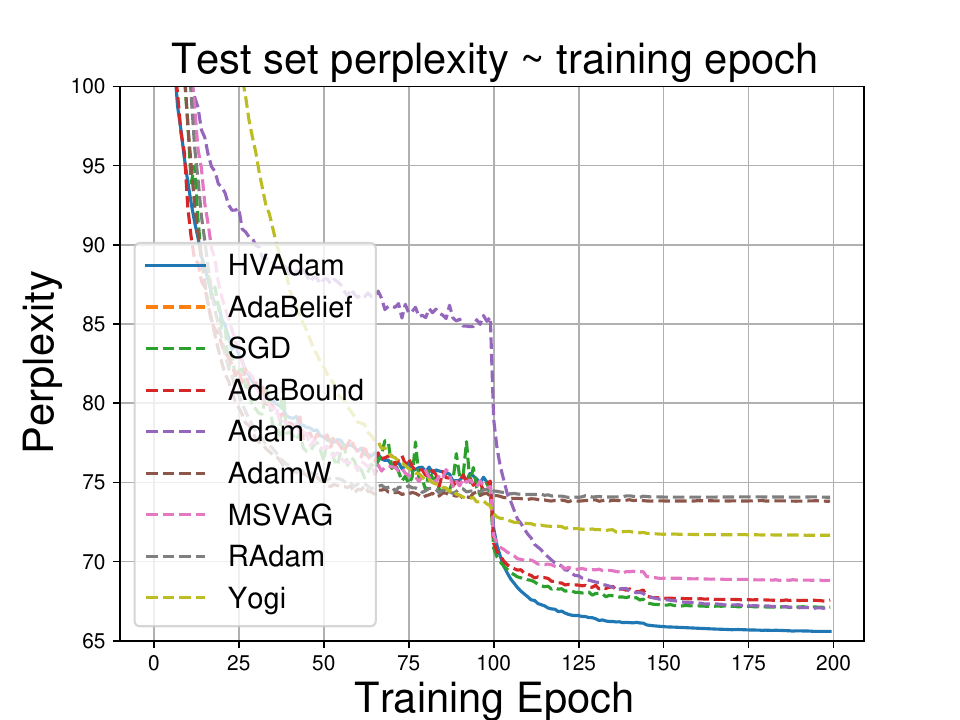}
    \end{subfigure}
    \begin{subfigure}[b]{0.32\textwidth}
    \includegraphics[width=\linewidth]{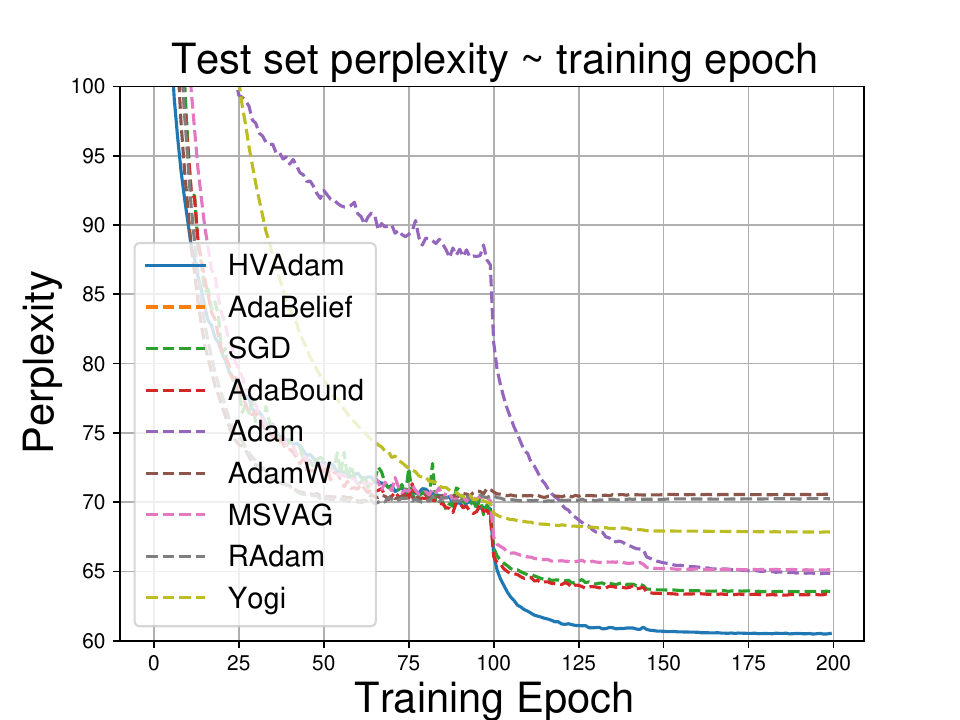}
    \end{subfigure}
    \caption{The perplexity on Penn Treebank for 1,2,3-layer LSTM from left to right. \textbf{Lower} is better.}
     \label{fig:LSTMtext}
\end{figure*}

\begin{table*}
\centering
\small

\centering

\begin{tabular}{c|c|c|c|c|c|c|c|c}
\hline
HVAdam               & AdaBelief               & RAdam          & RMSProp        & Adam           & Fromage        & Yogi           & SGD         & AdaBound                            \\ \hline
\textbf{12.72}$\pm$\textbf{0.21} & 12.98$\pm$0.22 & 13.10$\pm$0.20 & \underline{12.86$\pm$0.08} & 13.01$\pm$0.15 & 46.31$\pm$0.86 & 14.16$\pm$0.05 & 48.94$\pm$2.88  & \multicolumn{1}{c}{16.84$\pm$0.10} \\ \hline
\end{tabular}

\caption{FID values ($[\mu \pm \sigma]$) of a SN-GAN with ResNet generator on CIFAR-10. A lower FID value is better.}
\label{table:sngan_fid}



\end{table*}

\subsection{Experiments for Image Generation}
\begin{figure}
    \begin{subfigure}[b]{0.45\textwidth}
	\includegraphics[width=1\linewidth]{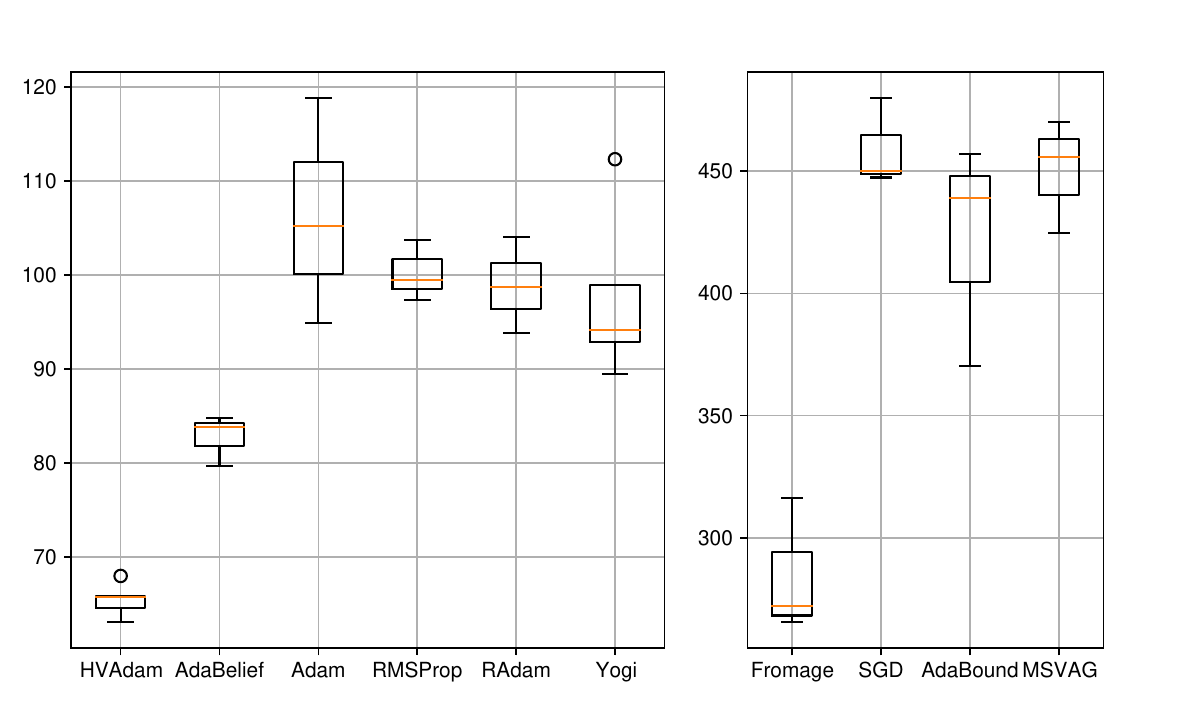}
	\caption{\small{FID scores of WGAN.}}	
	\end{subfigure}
    \begin{subfigure}[b]{0.45\textwidth}
	\includegraphics[width=1\linewidth]{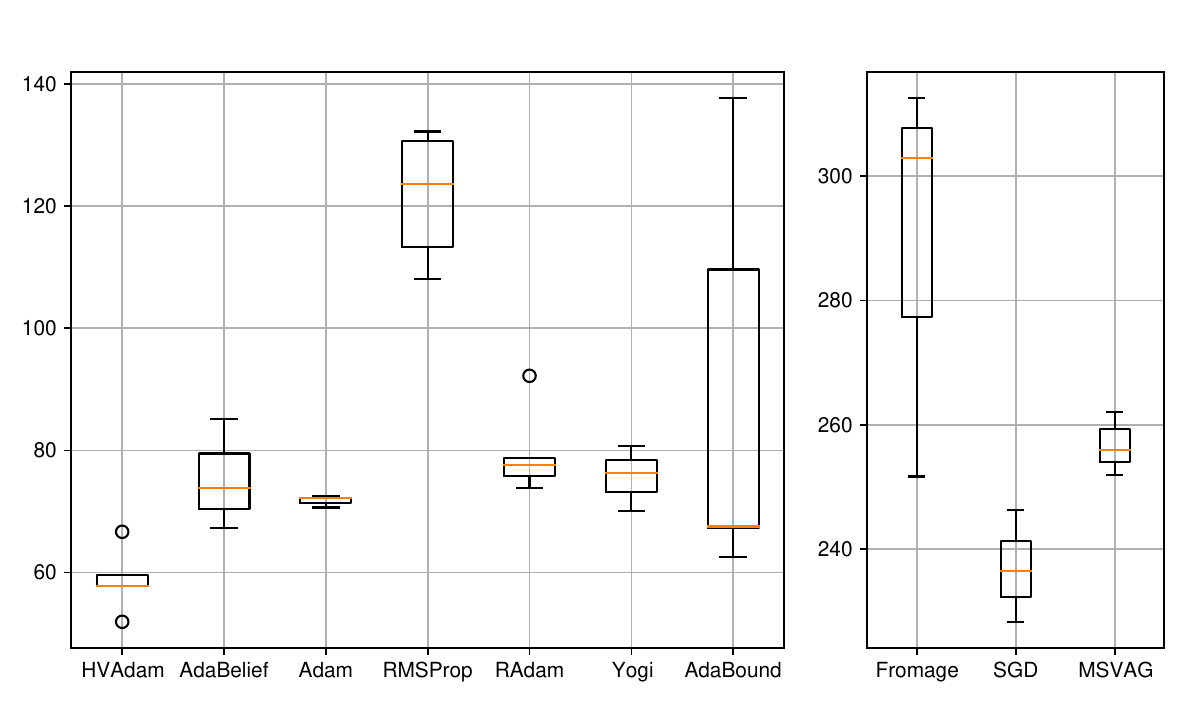}
	\caption{\small{FID scores of WGAN-GP.}}
	\end{subfigure}
    \caption{FID scores of different optimizers on WGAN and WGAN-GP (using the vanilla CNN generator on CIFAR-10). Lower FID indicates better performance. For every model, the successful and failed optimizers are displayed on the left and right sides, with different ranges of y values.}
    \label{fig:GAN}
\end{figure}

\paragraph{Generative adversarial networks (GANs) on image generation} 
Stability is also important for optimizers. 
And, as mentioned in \cite{salimans2016improved}, mode collapse and numerical instability can easily impact GAN training. So, training GANs can reflect the stability of optimizers. For example, SGD often fails when training GANs, whereas Adam can effectively train GANs. To assess the robustness of HVAdam, we performed experiments with WGAN, WGAN-GP, and SN-GAN using the CIFAR-10 dataset. The results were generated using the code from the \textit{official implementation} of AdaBelief and compared with the findings in \cite{buvanesh2021re}, where hyperparameters were explored more extensively.
We perform 5 runs of experiments, and The comparison results on WGAN, WGAN-GP, and SN-GAN are reported in Figure ~\ref{fig:GAN}, 
and Table~\ref{table:sngan_fid} . And we also add the great results of expriments on diffusion model \cite{nichol2021improved} in the supplementary material. We can see that HVAdam gets the lowest FID \cite{heusel2017gans} scores with all GANs, where the lower the FID is, the better the quality and diversity of the generated images. Furthermore, HVAdam's FID score with WGAN is even better than the other optimizers' FID scores with WGAN-GP. So, the stability of HVAdam is fully validated.



\subsection{Ablation Study} 
We perform the ablation study to demonstrate that every step of improvement plays important roles. We use three methods to train a 1-layer-LSTM. (a) Use Adam (HVAdam$^0$). (b) Use Adam which we introduce $v_t$ (HVAdam$^1$). (c) After introducing $v_t$, we can change Adam's adjustment strategy of learning rate by using $v_t$ to calculate $s_t$ (HVAdam$^2$). (d) Based on HVAdam$^2$, we introduce the restart strategy (HVAdam). As Table.~\ref{table:Ablation Study} shows, all the steps can help the optimizer perform better.

\begin{table}[h]
\centering
\setlength{\tabcolsep}{1mm}
\small
\begin{tabular}{c|c|c|c|c}
\hline
            &HVAdam$^0$ &HVAdam$^1$ &HVAdam$^2$ &HVAdam\\ \hline
valid ppl	&85.04	&84.54	&83.42	&83.31\\ \hline
\end{tabular}
\caption{The valid perplexity in ablation study.}
\label{table:Ablation Study}
\end{table}

\section{Conclusion}
We propose the HVAdam optimizer, which obtains the trend of the loss function through a hidden vector and restart strategy and utilizes the trend to update the learning rate, thereby accelerating the model's convergence speed. We validate HVAdam's advantages with four simple but representative functions and prove its convergence in both convex and non-convex cases. Experimental results show that HVAdam outperforms almost all other optimizers in tasks that cover a comprehensive category of deep learning tasks. 
Although our optimizer achieves excellent results and faster convergence speed in these experiments, the algorithm’s calculation process is relatively complex and incurs extra memory overhead. We recognize that these factors may limit the algorithm’s practicality in memory-constrained environments. In future work, we plan to optimize the algorithm’s computational efficiency and reduce its memory requirements to enhance its applicability.

\section{Acknowledgements}
This project was supported by the Key Research and Development Project of Hubei Province under grant number 2022BCA057. The numerical calculations in this paper have been done on the supercomputing system in the Supercomputing Center of Wuhan University.

\bibliographystyle{aaai25}
\bibliography{aaai25}

\clearpage
\onecolumn 


\appendix

\newtheorem{theorem}{Theorem}[section]
\newtheorem{corollary}{Corollary}[theorem]
\newtheorem{lemma}[theorem]{Lemma}
\setcounter{section}{0}
\setcounter{equation}{0}
\setcounter{theorem}{0}
\setcounter{table}{0}
\setcounter{figure}{0}
\allowdisplaybreaks

\section*{{\LARGE Appendix}}
\section{Detailed Algorithm of HVAdam}
\begin{itemize}[topsep=0pt,parsep=0pt,partopsep=0pt]
    \item $f(\theta) \in \mathbb{R}, \theta \in \mathbb{R}^d$: $f(\theta)$ is the scalar-valued function to minimize, $\theta$ is the parameter in $\mathbb{R}^d$ to be optimal.
    \item  $\prod_{\mathcal{F},S}(y) = \mathrm{argmin}_{x \in \mathcal{F}} \vert \vert S^{1/2} (x-y) \vert \vert$:
    The projection of $y$ onto convex feasible set $\mathcal{F}$.
    \item $g_t\in \mathbb{R}^d$: The gradient at step $t$.
    \item $m_t\in \mathbb{R}^d$: The exponential moving average (EMA) of $g_t$.
    \item $v_t\in \mathbb{R}^d$: The hidden vector calculated by $v_{t-1}$ and $m_t$.
	\item $h_t\in \mathbb{R}$: The EMA of $\Vert m_t\Vert^2$.
    \item $p_t\in \mathbb{R}^d$: $p_t=(g_t - v_{t-1})^2$. $p_t$ is an intermediate variable.
	\item $\eta_t\in \mathbb{R}^d$: The factor measure the size of $p_t$.
	\item $a_t\in \mathbb{R}^d$: The EMA of $g_t^2$.
	\item $s_t\in \mathbb{R}^d$: The EMA of $\eta_t p_t$.
    \item $\alpha_1, \alpha_2 ,\gamma\in \mathbb{R}$: $\alpha_1$ is the unadjusted learning rate for $m_t$; $\alpha_2$ is the unadjusted learning rate for $v_t$ ; $\gamma$ is a constant to limit the value of $\eta_t$, which is usually set as $0$. These are hyperparameters.
	\item $\delta_t\in \mathbb{R}$: The factor used to adjust $\alpha_{2t}$.
	\item $\epsilon \in \mathbb{R}$: The hyperparameter $\epsilon$ is a small number, used for avoiding division by 0.
\item $lr(\delta_{t_2},\widehat{\delta_{t_2}})\in \mathbb{R}\times \mathbb{R} \rightarrow \mathbb{R}$: $\left\{
        \begin{array}{ll}
            10^{\delta_{t_2}\cdot6-3} ,& \text{if }\widehat{\delta_{t_2}} \ge 0.1 \\
            0,   & \text{otherwise}
        \end{array}
    \right.
$. The function can be set as other more suitable choices; we will not change it in the rest of the paper unless otherwise specified.
    \item $\beta_1, \beta_2\in \mathbb{R}$: The hyperparameter for EMA, $0\le \beta_1,\beta_2 < 1$, typically set as 0.9 and 0.999.
\end{itemize}

\begin{algorithm}
\textbf{Input:} $\alpha_1$, $\beta_1$ , $\beta_2$, $\epsilon$, $\gamma$\\
\textbf{Initialize} $\theta_0$, $\alpha_1 \leftarrow \alpha_2$, $m_0 \leftarrow 0$ , $s_0 \leftarrow 0$, $v_0 \leftarrow 0$, $t \leftarrow 0$, $t_2 \leftarrow -1$, $\delta_0 \leftarrow 0$ 
\begin{algorithmic}[1]
\WHILE{$\theta_t$ not converged}
\STATE $t \leftarrow t + 1 $
\STATE $t_2 \leftarrow t_2 + 1 $
\STATE $g_t \leftarrow \nabla_{\theta}f_t(\theta_{t-1})$
\STATE $m_t \leftarrow \beta_1 m_{t-1} + (1 - \beta_1) g_t$\color{blue}
\STATE {$p_t \leftarrow {(g_t - v_{t-1})^2}$}
\STATE {$\eta_t \leftarrow \frac{p_t}{(g_t - m_t)^2 + \gamma p_t + \epsilon}$}
\STATE {$s_{t} \leftarrow \beta_2 s_{t-1} + (1 - \beta_2) {\eta_t p_t}+\epsilon$}\color{black}
\STATE \textbf{Bias Correction}
\STATE \hspace{3mm} $\widehat{m_t} \leftarrow \frac{m_t} {1-\beta_1^t}$,  $\widehat{s_t} \leftarrow \frac{s_t} {1-\beta_2^t}$
\IF{$t_2$ \textbf{not} $0$}\color{red}
\IF{$v_{t-1} = \widehat{m_t}$}
\STATE {$k_t=0$}
\ELSE
\STATE {$k_{t} \leftarrow \frac{\langle v_{t_2-1}-\widehat{m_t} , v_{t_2-1} \rangle}{\Vert v_{t_2-1}-\widehat{m_t} \Vert^2}$} 
\ENDIF \color{brown}
\STATE {$v_{t_2} \leftarrow {k_t \widehat{m_t}+(1-k_{t})v_{t_2-1}}$}
\STATE {$\delta_{t_2} \leftarrow \beta_2\delta_{t_2-1} + (1-\beta_2)\text{cos}\langle v_{t_2},\widehat{m_t}\rangle $}
\STATE {$\widehat{\delta_{t_2}} \leftarrow \frac{\delta_{t_2}}{1-\beta_2^{t_2}}$}
\STATE {$b_{t} \leftarrow lr(\delta_{t_2},\widehat{\delta_{t_2}})$}
\IF{$b_t = 0$}
\STATE{$t_2=-1$}
\ENDIF \color{black}
\ELSE
\STATE $v_{0} \leftarrow \widehat{m_t}$, $\delta_0 \leftarrow 0$
\ENDIF
\STATE \textbf{Update} \\
\STATE \hspace{3mm} $\theta_t \leftarrow \prod_{\mathcal{F},\sqrt{ \widehat{s_t}}} \Big( \theta_{t-1} -  \frac{\alpha_1 \widehat{m_t}} { \sqrt{{\widehat{s_t}}} + {\epsilon}} - \alpha_2 b_t v_t \Big)$
\ENDWHILE
\caption{HVAdam Optimizer}
\label{algo:HVAdam}
\end{algorithmic}
\end{algorithm}

\FloatBarrier

\FloatBarrier
\section{Convergence analysis about hidden vecter update algorithm}

Similar to \cite{reddi2019convergence, luo2019adaptive, chen2018convergence}, 
for simplicity, the de-biasing step is omitted, and it is the same for all subsequent analyses.
\begin{theorem}
\label{theorem:hidden_vector}
Under the assumptions:
\begin{itemize}[leftmargin=*]
    \item When $t<T^*$, $m_t=v^*+\zeta_t$, and $\langle v^*,\zeta_t \rangle=0$.
    
    \item The norm of $\zeta$ obeys truncated normal distribution and is unbiased, $i.e.\ f(\Vert \zeta \Vert) = \frac{\sqrt{2}}{\sigma \sqrt{\pi}}e^{-\frac{1}{2\sigma^2}\Vert \zeta \Vert^2},\sigma>0$.
\end{itemize}
The value of $v_t$ in the proposed algorithm satisfies: 
\begin{align*}
\Vert v_T\Vert =\operatorname*{min}_{t\in [T]}\Vert v_t \Vert \leq \operatorname*{min}_{t\in [T]}\Vert m_t \Vert \ \ \  \mathbb{E} \Vert v_T-v^* \Vert \leq \frac{\sigma \sqrt{2\pi}}{T}\nonumber
\end{align*}
It means that $v_t$ has a convergence rate of $O(1/T)$.
\end{theorem}
\textbf{\textit{Proof:}} Omit the small hyperparameter $\epsilon$, which is used to avoid dividing by 0. We have:
\begin{align}
\label{supeq:a-1}
\langle v_t,v_{t-1} \rangle &=\langle km_t+(1-k)v_{t-1},v_{t-1}\rangle \nonumber \\
	&=\Vert v_{t-1}-m_t\Vert^{-2}\langle \langle v_{t-1}-m_t , v_{t-1} \rangle m_t-\langle v_{t-1}-m_t , m_t \rangle v_{t-1} , v_{t-1}\rangle \nonumber \\
	&=\Vert v_{t-1}-m_t\Vert^{-2}(\langle v_{t-1}-m_t , v_{t-1} \rangle \langle m_t,v_{t-1} \rangle-\langle v_{t-1}-m_t , m_t \rangle \langle v_{t-1},v_{t-1}\rangle) \nonumber \\
	&=\Vert v_{t-1}-m_t\Vert^{-2}(\langle m_t , m_t \rangle \langle v_{t-1},v_{t-1} \rangle-\langle m_t , v_{t-1} \rangle \langle m_t,v_{t-1}\rangle)
\end{align}
and
\begin{align}
\label{supeq:a-2}
\langle v_t,m_t \rangle &=\Vert v_{t-1}-m_t\Vert^{-2}\langle \langle v_{t-1}-m_t , v_{t-1} \rangle m_t-\langle v_{t-1}-m_t , m_t \rangle v_{t-1} , m_t\rangle \nonumber \\
	&=\Vert v_{t-1}-m_t\Vert^{-2}(\langle v_{t-1}-m_t , v_{t-1} \rangle \langle m_t,m_t \rangle-\langle v_{t-1}-m_t , m_t \rangle \langle v_{t-1},m_t\rangle) \nonumber \\
	&=\Vert v_{t-1}-m_t\Vert^{-2}(\langle m_t , m_t \rangle \langle v_{t-1},v_{t-1} \rangle-\langle m_t , v_{t-1} \rangle \langle m_t,v_{t-1}\rangle)
\end{align}
From Formula ~\eqref{supeq:a-1} and Formula ~\eqref{supeq:a-2}, we have:
\begin{align}
\label{supeq:a-3}
\langle v_t,v_{t-1} \rangle &= \langle v_t,m_t \rangle \nonumber \\
(1-k)^{-1}\langle v_t,v_t-km_t \rangle &= \langle v_t,m_t \rangle \nonumber \\
\langle v_t,v_t \rangle -\langle v_t,km_t \rangle &= (1-k)\langle v_t,m_t \rangle \nonumber \\
\langle v_t,v_t \rangle &= \langle v_t,m_t \rangle \nonumber \\
\Vert v_t\Vert^2 &= \langle v_t,m_t \rangle \nonumber \\
\Vert v_t\Vert^2 &\leq \Vert v_t\Vert \Vert m_t \Vert \nonumber \\
(\textit{Cauchy-Schwartz's inequality: }& \langle u,v \rangle \leq \Vert u \Vert \Vert v \Vert )\nonumber \\
\Vert v_t\Vert &\leq \Vert m_t \Vert
\end{align}
Similarly, we have:
\begin{align}
\label{supeq:a-4}
\Vert v_t\Vert &\leq \Vert v_{t-1} \Vert
\end{align}
From Formula ~\eqref{supeq:a-3} and Formula ~\eqref{supeq:a-4}, we have:
\begin{align}
\label{supeq:a-5}
\Vert v_T\Vert =\operatorname*{min}_{t\in [T]}\Vert v_t \Vert \leq \operatorname*{min}_{t\in [T]}\Vert m_t \Vert
\end{align}
Since $m_t=v^*+\zeta_t$, and $v_t$ is the affine combination of $m_t$, we let $v_t=v^*+\zeta_{v,t}$, where $\langle v^*,\zeta_t \rangle=0$ and $\langle v^*,\zeta_{v,t} \rangle=0$. With Formula ~\eqref{supeq:a-5}, we have:
\begin{align}
\label{supeq:a-6}
\Vert v_T\Vert &\leq \operatorname*{min}_{t\in [T]}\Vert m_t \Vert \nonumber \\
\Vert v^*+\zeta_{v,T}\Vert &\leq \operatorname*{min}_{t\in [T]}\Vert v^*+\zeta_t \Vert \nonumber \\
\sqrt{\Vert v^*\Vert^2+\Vert \zeta_{v,T}\Vert^2} &\leq \operatorname*{min}_{t\in [T]}\sqrt{\Vert v^*\Vert^2+\Vert \zeta_t \Vert^2} \nonumber \\
\Vert \zeta_{v,T}\Vert &\leq \operatorname*{min}_{t\in [T]}\Vert \zeta_t \Vert \nonumber \\
\Vert v_T-v^*\Vert &\leq \operatorname*{min}_{t\in [T]}\Vert \zeta_t \Vert
\end{align}
$P(\Vert \zeta \Vert<\tau)=\int_{0}^{\tau} \frac{\sqrt{2}}{\sigma \sqrt{\pi}}e^{-\frac{1}{2\sigma^2}x^2}\, dx$,
which means at the step $T=\Big\lceil \frac{1}{\int_{0}^{\tau} \frac{\sqrt{2}}{\sigma \sqrt{\pi}}e^{-\frac{1}{2\sigma^2}x^2}\, dx}\Big\rceil$ (while $\tau \le \sqrt{3}\sigma$), we have (by ~\eqref{supeq:a-6}):
\begin{align}
\label{supeq:a0}
\mathbb{E}\Vert v_T-v^*\Vert \le \mathbb{E}\operatorname*{min}_{t\in [T]}\Vert \zeta_t \Vert \le \tau
\end{align}
And the at the step,
\begin{align}
\label{supeq:a1}
	T&=\Big\lceil\frac{1}{\int_{0}^{\tau} \frac{\sqrt{2}}{\sigma \sqrt{\pi}}e^{-\frac{1}{2\sigma^2}x^2}\, dx}\Big\rceil\nonumber \\
        &\leq\frac{1}{\int_{0}^{\tau} \frac{\sqrt{2}}{\sigma \sqrt{\pi}}e^{-\frac{1}{2\sigma^2}x^2}\, dx}\nonumber \\
	&\le \frac{1}{\int_{0}^{\tau}\frac{\sqrt{2}}{\sigma \sqrt{\pi}}( 1-\frac{x^2}{2\sigma^2})\, dx}\nonumber \\
	&= \frac{1}{\frac{\sqrt{2}}{\sigma \sqrt{\pi}}(\tau-\frac{\tau^3}{6\sigma^2})}\nonumber \\
	&\le \frac{1}{\frac{\sqrt{2}}{\sigma \sqrt{\pi}}(\tau-\frac{3\sigma^2\tau}{6\sigma^2})}\nonumber \\
	&= \frac{\sigma \sqrt{2\pi}}{\tau}
\end{align}
Apply Formula~\eqref{supeq:a1} to Formula~\eqref{supeq:a0}, we have:
\begin{align}
	\mathbb{E}\Vert v_T-v^*\Vert &\le \tau \le \frac{\sigma \sqrt{2\pi}}{T}\nonumber
\end{align}
So
\begin{align}
	\operatorname*{min}_{t\in [T]}\mathbb{E}\Vert v_T-v^* \Vert \le \mathbb{E}\Vert v_T-v^*\Vert  &\le \frac{\sigma \sqrt{2\pi}}{T}\nonumber
\end{align}

It implies the convergence rate for $v_t$ is $O(1 / T)$.
\hfill \qedsymbol

\FloatBarrier
\section{Convergence analysis in convex online learning case}
For simplicity, we omit the debiasing step in theoretical analysis as in \cite{reddi2019convergence}, and absorb $\epsilon$ into $s_t$. The analysis also applys to the de-biased version. 
\begin{lemma}{\cite{mcmahan2010adaptive}}
\label{suplemma:1}
For any $Q \in S^d_{+}$ and convex feasible set $\mathcal{F} \subset \mathbb{R}^d$, suppose $u_1 = \operatorname*{min}_{x \in \mathcal{F}} \Big \Vert Q^{1/2}(x-z_1) \Big \Vert$ and $u_2 = \operatorname*{min}_{x \in \mathcal{F}} \Big \Vert Q^{1/2}(x - z_2) \Big \Vert$, then we have $ \Big \Vert Q^{1/2}(u_1 - u_2)  \Big \Vert \leq \Big \Vert Q^{1/2} (z_1 - z_2)  \Big \Vert$.
\end{lemma}

\begin{theorem}
\label{suptheorem:2}
Let $\{\theta_t\}$ and $\{s_t\}$ be the sequence obtained by HVAdam, $\gamma>0$, $\beta_1 \in [0,1)$, $\beta_2 \in [0,1)$, $\beta_{1,t}\in [0,\beta_1]$, $\beta_{1,1}=\beta_1$, $\alpha_{1,t} = \frac{\alpha_1}{\sqrt{t}}$, $\alpha_{2,t}=\frac{\alpha_2}{t}$, $s_t \leq s_{t+1}$, $\forall t \in [T]$. 
 Assume that $\Vert x-y \Vert_\infty \leq D_\infty$, $\forall x,y\in \mathcal{F}$, where $\mathcal{F} \subset \mathbb{R}^d$ is a convex feasible set. Suppose $f(\theta)$ is a convex function, $ \Vert  g_t  \Vert_\infty \leq G_\infty / 2$ so that $ \Vert  g_t - v_t  \Vert_\infty \leq G_\infty$, $\forall t \in [T], \theta \in \mathcal{F}$. We also assume $s_{t,i} \geq c > 0$, $0\leq \delta_t \leq c_r$, $\forall t \in [T]$. The optimal point of $f$ is denoted by $\theta^*$. For $\{ \theta_t \}$ generated by HVAdam, there is a bound on the regret: 
\begin{align*}
\sum_{t=1}^T [f_t(\theta_t) - f_t(\theta^*)] &\leq \frac{D^2_\infty \sqrt{T}}{2(1-\beta_1)\alpha_1}\sum_{i=1}^ds_{T,i}^{1/2}
	+\frac{dG_\infty^2\alpha_1\sqrt{T}}{2\sqrt{c}(1-\beta_{1})}
	+\frac{dD_\infty^2G_\infty}{2(1 - \beta_{1})\alpha_1\sqrt{\gamma}}\sum_{t=1}^T \sqrt{t}\beta_{1,t}\nonumber \\
	&+\frac{3\alpha_{2}^2c_r^2dG_\infty^3}{8\alpha_{1}(1-\beta_{1})\sqrt{\gamma}}
	+\frac{\alpha_2c_rdG_\infty^{2}D_\infty \sqrt{T}}{\alpha_1(1-\beta_{1})\sqrt{\gamma}}+
\frac{\alpha_2c_rdG_\infty^{5/2}(1+\ln T)}{4(1-\beta_{1})\gamma^{1/4}c^{1/4}}
\end{align*}
\end{theorem} 
\textbf{\textit{Proof:}}
Let $\delta_t=$
\begin{equation*}
    \theta_{t+1} = \prod_{\mathcal{F}, \sqrt{s_t}}(\theta_t - \alpha_{1,t} s_t^{-1/2} m_t - \alpha_{2,t}\delta_t v_t) = \operatorname*{min}_{\theta \in \mathcal{F}} \Big \Vert s_t^{1/4}[\theta - (\theta_t - \alpha_{1,t} s_t^{-1/2}m_t- \alpha_{2,t}\delta_t v_t)] \Big \Vert
\end{equation*}
Note that $\prod_{\mathcal{F}, \sqrt{s_t}}(\theta^*) = \theta^*$ since $\theta^* \in \mathcal{F}$. Use $\theta^*_i$ and $\theta_{t,i}$ to denote the $i$th dimension of $\theta^*$ and $\theta_t$ respectively.
From Lemma \eqref{suplemma:1}, using $u_1 = \theta_{t+1}$ and $u_2 = \theta^*$, we have:\\
\begin{align}
\label{supeq:1}
    \Big \Vert s_t^{1/4}(\theta_{t+1} - \theta^*) \Big \Vert^2 & \leq \Big \Vert s_t^{1/4}(\theta_t - \alpha_{1,t} s_t^{-1/2}m_t - \alpha_{2,t}\delta_t v_t - \theta^*) \Big \Vert^2 \nonumber \\
    & =  \Big \Vert s_t^{1/4}(\theta_t - \theta^*) \Big \Vert^2 + \alpha_{1,t}^2 \Big \Vert s_t^{-1/4}m_t \Big \Vert^2 + \alpha_{2,t}^2\Big \Vert s_t^{1/4}\delta_t v_t \Big \Vert^2 \nonumber \\
	&- 2 \alpha_{1,t} \langle m_t, \theta_t - \theta^* \rangle - 2\alpha_{2,t}\langle s_t^{1/4}\theta_t - \alpha_{1,t}s_t^{-1/4}m_t  - s_t^{1/4}\theta^*,s_t^{1/4}\delta_t v_t \rangle \nonumber \\
    &= \Big \Vert s_t^{1/4}(\theta_t - \theta^*) \Big \Vert^2 + \alpha_{1,t}^2 \Big \Vert s_t^{-1/4}m_t \Big \Vert^2 +\alpha_{2,t}^2\Big \Vert s_t^{1/4}\delta_t v_t \Big \Vert^2 \nonumber \\
    &- 2 \alpha_{1,t} \langle \beta_{1,t}m_{t-1} + (1 - \beta_{1,t}) g_t, \theta_t - \theta^* \rangle \nonumber \\
	&- 2 \alpha_{2,t}\langle s_t^{1/4}\theta_t - \alpha_{1,t}s_t^{-1/4}m_t  - s_t^{1/4}\theta^*,s_t^{1/4}\delta_t v_t \rangle
\end{align}
Note that $\beta_1 \in [0,1)$ and $\beta_2 \in [0,1)$, by rearranging Inequality \eqref{supeq:1}, we have:
\begin{align}
\label{supeq:2}
    \langle g_t, \theta_t - \theta^* \rangle & \leq \frac{1}{2 \alpha_{1,t} (1 - \beta_{1,t})} \Big [ \Big \Vert s_t^{1/4}(\theta_t - \theta^*)\Big \Vert^2 - \Big \Vert s_t^{1/4} (\theta_{t+1} - \theta^*) \Big \Vert^2 \Big ] \nonumber \\
    &+ \frac{\alpha_{1,t}}{2(1-\beta_{1,t})} \Big \Vert s_t^{-1/4} m_t\Big \Vert^2 + \frac{\alpha_{2,t}^2}{2\alpha_{1,t}(1-\beta_{1,t})}\Big \Vert s_t^{1/4}\delta_t v_t \Big \Vert^2 \nonumber \\
    &+ \frac{\beta_{1,t}}{1 - \beta_{1,t}} \langle m_{t-1}, \theta^* -\theta_t \rangle \nonumber \\
	&+ \frac{\alpha_{2,t}}{\alpha_{1,t}(1-\beta_{1,t})}\langle -s_t^{1/4}\theta_t + \alpha_{1,t}s_t^{-1/4}m_t  + s_t^{1/4}\theta^*,s_t^{1/4}\delta_t v_t \rangle \nonumber \\
	& \leq \frac{1}{2 \alpha_{1,t} (1 - \beta_{1,t})} \Big [ \Big \Vert s_t^{1/4}(\theta_t - \theta^*)\Big \Vert^2 - \Big \Vert s_t^{1/4} (\theta_{t+1} - \theta^*) \Big \Vert^2 \Big ] \nonumber \\
    &+ \frac{\alpha_{1,t}}{2(1-\beta_{1,t})} \Big \Vert s_t^{-1/4} m_t\Big \Vert^2 + \frac{\alpha_{2,t}^2}{2\alpha_{1,t}(1-\beta_{1,t})}\Big \Vert s_t^{1/4}\delta_t v_t \Big \Vert^2 \nonumber \\
    &+ \frac{\beta_{1,t}}{2(1-\beta_{1,t})} \alpha_t \Big \Vert s_t^{-1/4} m_{t-1} \Big \Vert^2 + \frac{\beta_{1,t}}{2 \alpha_t (1 - \beta_{1,t})} \Big \Vert s_t^{1/4} (\theta_t - \theta^*) \Big \Vert^2 \nonumber \\
	&+ \frac{\alpha_{2,t}}{\alpha_{1,t}(1-\beta_{1,t})}\Big \Vert  s_t^{1/4}\theta_t - \alpha_{1,t}s_t^{-1/4}m_t  - s_t^{1/4}\theta^* \Big \Vert \Big \Vert s_t^{1/4}\delta_t v_t \Big \Vert \nonumber \\
    &\Big (\textit{Cauchy-Schwartz's inequality: } \langle u,v \rangle \leq \Big \Vert u \Big \Vert \Big \Vert v \Big \Vert \Big )\nonumber \\
	&\leq \frac{1}{2 \alpha_{1,t} (1 - \beta_{1,t})} \Big [ \Big \Vert s_t^{1/4}(\theta_t - \theta^*)\Big \Vert^2 - \Big \Vert s_t^{1/4} (\theta_{t+1} - \theta^*) \Big \Vert^2 \Big ] \nonumber \\
    &+ \frac{\alpha_{1,t}}{2(1-\beta_{1,t})} \Big \Vert s_t^{-1/4} m_t\Big \Vert^2 + \frac{\alpha_{2,t}^2}{2\alpha_{1,t}(1-\beta_{1,t})}\Big \Vert s_t^{1/4}\delta_t v_t \Big \Vert^2 \nonumber \\
    &+ \frac{\beta_{1,t} \alpha_{1,t}}{2(1-\beta_{1,t})}\Big \Vert s_t^{-1/4} m_{t-1} \Big \Vert^2 + \frac{\beta_{1,t}}{2 \alpha_{1,t} (1 - \beta_{1,t})} \Big \Vert s_t^{1/4} (\theta_t - \theta^*) \Big \Vert^2 \nonumber \\
	&+ \frac{\alpha_{2,t}}{\alpha_{1,t}(1-\beta_{1,t})}\Big [\Big \Vert  s_t^{1/4}(\theta_t-\theta^*)\Big \Vert +\Big \Vert \alpha_{1,t}s_t^{-1/4}m_t\Big \Vert  \Big ]\Big \Vert s_t^{1/4}\delta_t v_t \Big \Vert \nonumber \\
	&\Big (\textit{Minkowski's inequality: } \Big \Vert x+y\Big \Vert_p\leq \Big \Vert x \Big \Vert_p+\Big \Vert y\Big \Vert_p \Big)
\end{align}
By convexity of $f$, we have:
\begin{align}
\label{supeq:F}
\sum_{t=1}^T f_t(\theta_t) - f_t(\theta^*) & \leq \sum_{t=1}^T\langle g_t, \theta_t - \theta^* \rangle \nonumber \\
& \leq \sum_{t=1}^T \Big \{ \frac{1}{2 \alpha_{1,t} (1 - \beta_{1,t})} \Big [ \Big \Vert s_t^{1/4}(\theta_t - \theta^*)\Big \Vert^2 - \Big \Vert s_t^{1/4} (\theta_{t+1} - \theta^*) \Big \Vert^2 \Big ] \nonumber \\
    &+ \frac{\alpha_{1,t}}{2(1-\beta_{1,t})} \Big \Vert s_t^{-1/4} m_t\Big \Vert^2 + \frac{\alpha_{2,t}^2}{2\alpha_{1,t}(1-\beta_{1,t})}\Big \Vert s_t^{1/4}\delta_t v_t \Big \Vert^2 \nonumber \\
    &+ \frac{\beta_{1,t} \alpha_{1,t}}{2(1-\beta_{1,t})}\Big \Vert s_t^{-1/4} m_{t-1} \Big \Vert^2 + \frac{\beta_{1,t}}{2 \alpha_{1,t} (1 - \beta_{1,t})} \Big \Vert s_t^{1/4} (\theta_t - \theta^*) \Big \Vert^2 \nonumber \\
	&+ \frac{\alpha_{2,t}}{\alpha_{1,t}(1-\beta_{1,t})}\Big [\Big \Vert  s_t^{1/4}(\theta_t-\theta^*)\Big \Vert +\Big \Vert \alpha_{1,t}s_t^{-1/4}m_t\Big \Vert  \Big ]\Big \Vert s_t^{1/4}\delta_t v_t \Big \Vert  \Big \} \nonumber \\
	&\Big( By\  formula ~\eqref{supeq:2} \Big) \nonumber \\
	& \leq \underbrace{\sum_{t=1}^T \Big \{\frac{1}{2 \alpha_{1,t} (1 - \beta_{1,t})} \Big [ \Big \Vert s_t^{1/4}(\theta_t - \theta^*)\Big \Vert^2 - \Big \Vert s_t^{1/4} (\theta_{t+1} - \theta^*) \Big \Vert^2 \Big ]\Big \}}_P\nonumber \\
	& + \underbrace{\sum_{t=1}^T \Big [ \frac{\alpha_{1,t}(1+\beta_{1,t})}{2(1-\beta_{1,t})}\Big \Vert s_t^{-1/4} m_t\Big \Vert^2\Big ] }_Q + \underbrace{\sum_{t=1}^T \Big [ \frac{\beta_{1,t}}{2 \alpha_{1,t} (1 - \beta_{1,t})} \Big \Vert s_t^{1/4} (\theta_t - \theta^*) \Big \Vert^2    \Big ] }_R \nonumber \\
	& + \underbrace{\sum_{t=1}^T \Big [
	\frac{\alpha_{2,t}^2}{2\alpha_{1,t}(1-\beta_{1,t})}\Big \Vert s_t^{1/4}\delta_t v_t \Big \Vert^2
	\Big ] }_S \nonumber \\
	& + \underbrace{\sum_{t=1}^T \Big \{
	\frac{\alpha_{2,t}}{\alpha_{1,t}(1-\beta_{1,t})}\Big [\Big \Vert  s_t^{1/4}(\theta_t-\theta^*)\Big \Vert +\Big \Vert \alpha_{1,t}s_t^{-1/4}m_t\Big \Vert  \Big ]\Big \Vert s_t^{1/4}\delta_t v_t \Big \Vert \Big \} }_U
\end{align}
Firstly, bound $P$ in Formula ~\eqref{supeq:F}, assuming $0 < c \leq s_t, \forall t \in [T]$.
\begin{align}
\label{supeq:P}
    P&\leq \frac{1}{2(1-\beta_1)}\sum_{t=1}^T\Big \{ \frac{1}{\alpha_{1,t}}\Big [ \Big \Vert s_t^{1/4}(\theta_t - \theta^*)\Big \Vert^2 - \Big \Vert s_t^{1/4} (\theta_{t+1} - \theta^*) \Big \Vert^2 \Big ]
\Big \}\nonumber \\
	& \Big( \textit{since\ } 0\leq \beta_{1,t}\leq \beta_1<1 \Big)\nonumber \\
	& \leq \frac{1}{2(1-\beta_1)\alpha_1}\Big \Vert s_1^{1/4}(\theta_1 - \theta^*)\Big \Vert^2+ \frac{1}{2(1-\beta_1)}\sum_{t=2}^T\sum_{i=1}^d\Big(\theta_{t,i} - \theta^*_i\Big)^2 \Big [\frac{s_{t,i}^{1/2}}{\alpha_{1,t}} -\frac{\ s_{t-1,i}^{1/2}}{\alpha_{1,t-1}}\Big ] \nonumber \\
	& \leq \frac{D^2_\infty}{2(1-\beta_1)\alpha_1} \sum_{i=1}^ds_{1,i}^{1/2}+ \frac{D^2_\infty}{2(1-\beta_1)}\sum_{t=2}^T\sum_{i=1}^d\Big [\frac{s_{t,i}^{1/2}}{\alpha_{1,t}} -\frac{ s_{t-1,i}^{1/2}}{\alpha_{1,t-1}}\Big ] \nonumber \\
	& \Big( \textit{since\ }\frac{s_{t,i}^{1/2}}{\alpha_{1,t}}\geq \frac{s_{t-1,i}^{1/2}}{\alpha_{1,t-1}} \Big)\nonumber \\
	& = \frac{D^2_\infty}{2(1-\beta_1)\alpha_{1,t}}\sum_{i=1}^d s_{T,i}^{1/2}\nonumber \\
	& = \frac{D^2_\infty \sqrt{T}}{2(1-\beta_1)\alpha_1}\sum_{i=1}^ds_{T,i}^{1/2}
\end{align}
Secondly, bound $Q$ in Formula ~\eqref{supeq:F}.
\begin{align}
\label{supeq:Q}
    Q&\leq \sum_{t=1}^T \Big [ \frac{\alpha_{1,t}(1+\beta_{1,t})}{2\sqrt{c}(1-\beta_{1,t})}\Big \Vert m_t\Big \Vert^2\Big ] \nonumber \\
	&\Big( \textit{since\ } s_t\geq c>0 \Big)\nonumber \\
	&\leq \frac{1}{\sqrt{c}(1-\beta_{1})}\sum_{t=1}^T \alpha_{1,t}\Big \Vert m_t\Big \Vert^2\nonumber \\
	& \Big( \textit{since\ } 0\leq \beta_{1,t}\leq \beta_1<1 \Big)\nonumber \\
	&\leq \frac{dG_\infty^2}{4\sqrt{c}(1-\beta_{1})}\sum_{t=1}^T \frac{\alpha_{1}}{\sqrt{t}}\nonumber \\
	&\leq \frac{dG_\infty^2\alpha_1}{4\sqrt{c}(1-\beta_{1})}\int_{0}^T \frac{1}{\sqrt{t}}\, dt \nonumber \\
	&\leq \frac{dG_\infty^2\alpha_1\sqrt{T}}{2\sqrt{c}(1-\beta_{1})}
\end{align}
Thirdly, bound $R$ in Formula ~\eqref{supeq:F}.
\begin{align}
\label{supeq:R}
	R&\leq \frac{1}{2(1 - \beta_{1})}\sum_{t=1}^T \Big [ \frac{\beta_{1,t}}{\alpha_{1,t}} \Big \Vert s_t^{1/4} (\theta_t - \theta^*) \Big \Vert^2    \Big ] \nonumber \\
	&\leq \frac{D_\infty^2}{2(1 - \beta_{1})}\sum_{t=1}^T\frac{\beta_{1,t}}{\alpha_{1,t}}\sum_{i=1}^d  s_{t,i}^{1/2}  \nonumber \\
	&\leq \frac{D_\infty^2}{2(1 - \beta_{1})\alpha_1}\sum_{t=1}^T\sum_{i=1}^d  \sqrt{t}\beta_{1,t}s_{t,i}^{1/2}  \nonumber \\
	&\leq \frac{dD_\infty^2G_\infty}{2(1 - \beta_{1})\alpha_1\sqrt{\gamma}}\sum_{t=1}^T \sqrt{t}\beta_{1,t}  \nonumber \\
	&\Big(\textit{since\ } s_t \textit{\ is the EMA of\ } \eta_t(g_t-v_t)^2 \textit{\ and \ } \eta_t\in (0,\frac{1}{\gamma}) \Big) 
\end{align}
Fourthly, bound $S$ in Formula ~\eqref{supeq:F}.
\begin{align}
\label{supeq:S}
	S&\leq \sum_{t=1}^T \frac{\alpha_{2,t}^2}{2\alpha_{1,t}(1-\beta_{1,t})}\sum_{i=1}^d \frac{G_\infty^2}{4} s_{t,i}^{1/2}\delta_{t}^2\nonumber \\
	&\leq  \frac{\alpha_{2}^2G_\infty^2}{8\alpha_{1}(1-\beta_{1})}\sum_{t=1}^T\sum_{i=1}^d t^{-3/2} s_{t,i}^{1/2}\delta_{t}^2\nonumber \\
	&\leq  \frac{\alpha_{2}^2c_r^2G_\infty^2}{8\alpha_{1}(1-\beta_{1})}\sum_{t=1}^T\sum_{i=1}^d t^{-3/2} s_{t,i}^{1/2}\nonumber \\
	& \Big( \textit{since\ } 0\leq \delta_{t} \leq c_r \Big)\nonumber \\
	&\leq  \frac{\alpha_{2}^2c_r^2dG_\infty^3}{8\alpha_{1}(1-\beta_{1})\sqrt{\gamma}}\sum_{t=1}^Tt^{-3/2}\nonumber \\
	&=  \frac{\alpha_{2}^2c_r^2dG_\infty^3}{8\alpha_{1}(1-\beta_{1})\sqrt{\gamma}}\Big[1+\sum_{t=2}^Tt^{-3/2}\Big]\nonumber \\
	&\leq  \frac{\alpha_{2}^2c_r^2dG_\infty^3}{8\alpha_{1}(1-\beta_{1})\sqrt{\gamma}}\Big [1+\int_{1}^{T}t^{-3/2}\, dt\Big ]\nonumber \\
	&=  \frac{\alpha_{2}^2c_r^2dG_\infty^3}{8\alpha_{1}(1-\beta_{1})\sqrt{\gamma}}\Big[ 1+2-\frac{2}{\sqrt{T}}\Big]\nonumber \\
	&\leq  \frac{3\alpha_{2}^2c_r^2dG_\infty^3}{8\alpha_{1}(1-\beta_{1})\sqrt{\gamma}}
\end{align}
Fifthly, bound $U$ in Formula ~\eqref{supeq:F}.
\begin{align}
\label{supeq:U}
	U&\leq \frac{\alpha_2}{\alpha_1(1-\beta_{1})}\sum_{t=1}^T \frac{1}{\sqrt{t}}
	\Big [\Big \Vert  s_t^{1/4}(\theta_t-\theta^*)\Big \Vert +\Big \Vert \alpha_{1,t}s_t^{-1/4}m_t\Big \Vert  \Big ]\Big \Vert s_t^{1/4}\delta_t v_t \Big \Vert \nonumber \\
	&\leq \frac{\alpha_2c_r\sqrt{d}G_\infty^{3/2}}{2\alpha_1(1-\beta_{1})\gamma^{1/4}}\sum_{t=1}^T \frac{1}{\sqrt{t}}
	\Big [\Big \Vert  s_t^{1/4}(\theta_t-\theta^*)\Big \Vert +\Big \Vert \alpha_{1,t}s_t^{-1/4}m_t\Big \Vert  \Big ] \nonumber \\
	&\leq \frac{\alpha_2c_rdG_\infty^{3/2}}{2\alpha_1(1-\beta_{1})\gamma^{1/4}}\sum_{t=1}^T \frac{1}{\sqrt{t}}
	\Big [\frac{G_\infty^{1/2}D_\infty}{\gamma^{1/4}} +\frac{\alpha_1G_\infty}{2\sqrt{t}c^{1/4}}  \Big ] \nonumber \\
	&= \frac{\alpha_2c_rdG_\infty^{2}D_\infty}{2\alpha_1(1-\beta_{1})\gamma^{1/2}}\sum_{t=1}^T \frac{1}{\sqrt{t}}+
\frac{\alpha_2c_rdG_\infty^{5/2}}{4(1-\beta_{1})\gamma^{1/4}c^{1/4}}\sum_{t=1}^T \frac{1}{t}  \nonumber \\
	&\leq \frac{\alpha_2c_rdG_\infty^{2}D_\infty}{2\alpha_1(1-\beta_{1})\gamma^{1/2}}\Big[\int_{0}^T\frac{1}{\sqrt{T}}\, dt\Big]+
\frac{\alpha_2c_rdG_\infty^{5/2}}{4(1-\beta_{1})\gamma^{1/4}c^{1/4}}\Big[1+\int_{1}^T \frac{1}{t}\, dt\Big]\nonumber \\
	&\leq \frac{\alpha_2c_rdG_\infty^{2}D_\infty \sqrt{T}}{\alpha_1(1-\beta_{1})\sqrt{\gamma}}+
\frac{\alpha_2c_rdG_\infty^{5/2}(1+\ln T)}{4(1-\beta_{1})\gamma^{1/4}c^{1/4}}
\end{align}

Apply formulas~\eqref{supeq:P},~\eqref{supeq:Q},~\eqref{supeq:R},~\eqref{supeq:S} and~\eqref{supeq:U} to~\eqref{supeq:F}, we have:
\begin{align}
\label{supeq:F1}
	\sum_{t=1}^T f_t(\theta_t) - f_t(\theta^*) & \leq P+Q+R+S+U\nonumber \\
	&\leq \frac{D^2_\infty \sqrt{T}}{2(1-\beta_1)\alpha_1}\sum_{i=1}^ds_{T,i}^{1/2}
	+\frac{dG_\infty^2\alpha_1\sqrt{T}}{2\sqrt{c}(1-\beta_{1})}
	+\frac{dD_\infty^2G_\infty}{2(1 - \beta_{1})\alpha_1\sqrt{\gamma}}\sum_{t=1}^T \sqrt{t}\beta_{1,t}\nonumber \\
	&+\frac{3\alpha_{2}^2c_r^2dG_\infty^3}{8\alpha_{1}(1-\beta_{1})\sqrt{\gamma}}
	+\frac{\alpha_2c_rdG_\infty^{2}D_\infty \sqrt{T}}{\alpha_1(1-\beta_{1})\sqrt{\gamma}}+
\frac{\alpha_2c_rdG_\infty^{5/2}(1+\ln T)}{4(1-\beta_{1})\gamma^{1/4}c^{1/4}}
\end{align} \hfill \qedsymbol

\begin{corollary}
Suppose $\beta_{1,t} = \beta_1 \lambda^t,\ \ 0<\lambda<1$ in Theorem  \eqref{suptheorem:2}, then we have:
\begin{align}
\sum_{t=1}^T f_t(\theta_t) - f_t(\theta^*) & \leq \frac{D^2_\infty \sqrt{T}}{2(1-\beta_1)\alpha_1}\sum_{i=1}^ds_{T,i}^{1/2}
	+\frac{dG_\infty^2\alpha_1\sqrt{T}}{2\sqrt{c}(1-\beta_{1})}
	+\frac{dD_\infty^2G_\infty \beta_1}{2(1 - \beta_{1})\alpha_1\sqrt{\gamma}(1-\lambda)^2}\nonumber \\
	&+\frac{3\alpha_{2}^2c_r^2dG_\infty^3}{8\alpha_{1}(1-\beta_{1})\sqrt{\gamma}}
	+\frac{\alpha_2c_rdG_\infty^{2}D_\infty \sqrt{T}}{\alpha_1(1-\beta_{1})\sqrt{\gamma}}+
\frac{\alpha_2c_rdG_\infty^{5/2}(1+\ln T)}{4(1-\beta_{1})\gamma^{1/4}c^{1/4}}
\end{align}
\end{corollary}
\textbf{ \textit{Proof:\ }} By sum of arithmetico-geometric series, we have:
\begin{align}
\label{supeq:geometric}
    \sum_{t=1}^T \lambda^{t-1}\sqrt{t} \leq \sum_{t=1}^T\lambda^{t-1} t \leq \frac{1}{(1-\lambda)^2}
\end{align}
Plugging~\eqref{supeq:geometric} into~\eqref{supeq:F1}, we can derive the results above. 
\hfill \qedsymbol

\section{Convergence analysis for non-convex stochastic optimization }
\begin{theorem}The update of $\theta_t$ can be described as $\theta_{t+1}=\theta_t-\alpha_{1,t}A_tm_t-\alpha_{2,t}B_tv_t$, and $m_t=\beta_1m_{t-1}+(1-\beta_1)g_t$. $A_t=\frac{1}{\sqrt{s_t}+\epsilon}$, $B_t=\delta_t$. The hyperparameters are set as: $\alpha_{1,t}=\alpha_0t^{-k}$, $\alpha_{2,t}=\alpha_0t^{-2k}$, $\alpha_0\leq \frac{C_l}{LC_r^2}$, $k\in [0.5,1)$. The bounds are $C_lI\preceq A_t\preceq C_rI$, $0I \preceq B_t \preceq C_rI$ and $0<C_l<C_r$ ($A\preceq B$ means $B-A$ is a positive semi-definite matrix). And the $\epsilon$ and $N$ ensure $C_l$ and $C_r$ exist. Assume $f$ is upper bounded by $M_f$. Denote $\frac{t^{-k}}{1-\beta_1}A_t^{-1}B_tv_t-\frac{{(t-1)}^{-k}\beta_1}{1-\beta_1}A_{t-1}^{-1}B_{t-1}v_{t-1}$ as $H_t$, and let $H_1=\frac{1}{1-\beta_1}A_1^{-1}B_1v_1$\\
Under the assumptions:
\begin{itemize}[leftmargin=*]
    \item $f$ is differentiable and $f^*\leq f \leq F$. $\nabla f(x)$ is L-Lipschitz continuous, $i.e.\ \ \Vert \nabla f(x) - \nabla f(y)  \Vert \leq L  \Vert x-y  \Vert,\ \forall x,y$.
    
    \item The noisy gradient is unbias and its norm is bounded by N which also bounds the gradient's norm. $i.e.\ \ \mathbb{E} g_t = \nabla f(x)$, $\Vert g_t \Vert_\infty \leq N$, $\Vert f_t \Vert_\infty \leq N$.


\end{itemize}
For sequence $\{\theta_t\}$ generated by HVAdam, we have: 
\begin{align*}
\frac{1}{T}\sum \limits^T_{t=1}\Big \Vert \nabla f(\theta) \Big \Vert^2 \leq \frac{1}{\alpha_0 C_l} T^{k -1} \Big[ F - f^* + JN^2\Big(1+\frac{1}{1-\beta_1}C_l^{-1}C_r\Big)^2 \Big]+3T^{-k}\frac{dC_l^{-1}C_rN^2}{(1-\beta_1)(1-k)}
\end{align*}
where
\begin{align*}
J=\frac{\beta_1^2}{4L (1-\beta_1)^2} + \frac{1}{1-\beta_1} \alpha_0 C_r + \Big( \frac{\beta_1}{1-\beta_1}+\frac{1}{2} \Big) L\alpha_0^2C_u^2 \Big(1+ \int_{1}^{T} t^{-2p}\, \mathrm{d}t\Big) \nonumber
\end{align*}
It shows that when $k=0.5$, HVAdam has a convergence rate of $O(\log T/\sqrt{T})$.
\end{theorem}
\textbf{ \textit{Proof: }} Let $m_{h,1}=\beta_1m_{h,t-1}+(1-\beta_1)g_{h,t}$, $g_{h,t}=g_t+H_t$, then we have:
\begin{align}
\theta_{t+1}&=\theta_t-\alpha_{1,t}A_tm_t-\alpha_{2,t}B_tv_t \nonumber \\
	&=\theta_t-\alpha_{1,t}A_t(m_t+t^{-k}A_t^{-1}B_tv_t) \nonumber \\
	&\Big( \textit{since\ } \alpha_{1,t}=\alpha_0t^{-k}, \alpha_{2,t}=\alpha_0t^{-2k} \Big)\nonumber \\
	&=\theta_t-\alpha_{1,t}A_t\Big[(1-\beta_1)\sum_{i=1}^t\beta_1^{t-i}g_i+t^{-k}A_t^{-1}B_tv_t\Big] \nonumber \\
	&=\theta_t-\alpha_{1,t}A_t(1-\beta_1)\Big[\sum_{i=1}^t\beta_1^{t-i}g_i+\sum_{i=1}^t\frac{\beta_1^{t-i}i^{-k}}{1-\beta_1}A_i^{-1}B_iv_i - \sum_{i=1}^{t-1}\frac{\beta_1^{t-i}i^{-k}}{1-\beta_1}A_{i}^{-1}B_{i}v_{i}\Big] \nonumber \\
	&=\theta_t-\alpha_{1,t}A_t(1-\beta_1)\Big[\sum_{i=1}^t\beta_1^{t-i}g_i+\sum_{i=1}^t\frac{\beta_1^{t-i}i^{-k}}{1-\beta_1}A_i^{-1}B_iv_i - \sum_{i=2}^{t}\frac{\beta_1^{t-(i-1)}(i-1)^{-k}}{1-\beta_1}A_{i-1}^{-1}B_{i-1}v_{i-1}\Big] \nonumber \\
	&=\theta_t-\alpha_{1,t}A_t(1-\beta_1)\Big[\sum_{i=1}^t\beta_1^{t-i}g_i+\sum_{i=1}^t\frac{\beta_1^{t-i}i^{-k}}{1-\beta_1}A_i^{-1}B_iv_i - \sum_{i=2}^{t}\frac{\beta_1^{t-i}(i-1)^{-k}}{1-\beta_1}\beta_1A_{i-1}^{-1}B_{i-1}v_{i-1}\Big] \nonumber \\
	&=\theta_t-\alpha_{1,t}A_t(1-\beta_1)\Big[\sum_{i=1}^t\beta_1^{t-i}g_i+\sum_{i=2}^t\frac{\beta_1^{t-i}i^{-k}}{1-\beta_1}A_i^{-1}B_iv_i - \sum_{i=2}^{t}\frac{\beta_1^{t-i}(i-1)^{-k}}{1-\beta_1}\beta_1A_{i-1}^{-1}B_{i-1}v_{i-1} \nonumber \\
    &+\frac{\beta_1^{t-1}}{1-\beta_1}A_1^{-1}B_1v_1\Big] \nonumber \\
	&=\theta_t-\alpha_{1,t}A_t(1-\beta_1)\Bigg\{\sum_{i=1}^t\beta_1^{t-i}g_i+\sum_{i=2}^t\beta_1^{t-i}\Big[\frac{i^{-k}}{1-\beta_1}A_i^{-1}B_iv_i - \frac{(i-1)^{-k}}{1-\beta_1}\beta_1A_{i-1}^{-1}B_{i-1}v_{i-1}\Big] \nonumber \\
    &+\frac{\beta_1^{t-1}}{1-\beta_1}A_1^{-1}B_1v_1\Bigg\} \nonumber \\
	&=\theta_t-\alpha_{1,t}A_t(1-\beta_1)\Big[\sum_{i=1}^t\beta_1^{t-i}g_i+\sum_{i=2}^t\beta_1^{t-i}H_i+\beta_1^{t-1}H_1\Big] \nonumber \\
	&=\theta_t-\alpha_{1,t}A_t(1-\beta_1)\sum_{i=1}^t\beta_1^{t-i}\Big[g_i+H_i\Big] \nonumber \\
	&=\theta_t-\alpha_{1,t}A_t(1-\beta_1)\sum_{i=1}^t\beta_1^{t-i}g_{h,i} \nonumber \\
	&=\theta_t-\alpha_{1,t}A_tm_{h,t}
\end{align}
Let $\alpha_t=\alpha_{1,t}$. And under the above assumptions, now we have:
\begin{align}
&\theta_{t+1}=\theta_t-\alpha_{t}A_tm_{h,t}\\
&m_{h,1}=\beta_1m_{h,t-1}+(1-\beta_1)g_{h,t}\\
&\mathbb{E}g_{h,t} =\mathbb{E}[g_t+H_t]
	=\mathbb{E}g_t+H_t
	=\nabla f(x)+\mathbb{E}H_t\\
&\Vert g_{h,t} \Vert_\infty =\Vert g_t+H_t \Vert_\infty \leq N+\frac{2(t-1)^{-k}}{1-\beta_1}C_l^{-1}C_rN=N+P\\
&\Big(\textit{Let }P=\frac{2(t-1)^{-k}}{1-\beta_1}C_l^{-1}C_rN\Big)\nonumber \\
&\mathbb{E}\big[\Vert g_{h,t}-\nabla f(x)\Vert^2\big]=\mathbb{E}\big[\Vert g_t + H_t-\nabla f(x)\Vert^2\big]=2\mathbb{E}\big[\Vert g_t-\nabla f(x)\Vert^2 + \Vert H_t\Vert^2\big]\nonumber \\
	&\leq 2\sigma^2+2P^2
\end{align}
Refer to \cite{zhuang2021momentum} for the rest of the proof. At last we have:
\begin{align}
\frac{1}{T}\sum \limits^T_{t=1}\Big \Vert \nabla f(\theta) \Big \Vert^2 \leq \frac{1}{\alpha_0C_l}T^{k-1}\Big[ F -f^*+J(D+N)^2\Big]
\end{align}
where
\begin{align}
J=\frac{\beta_1^2}{4L(1-\beta_1)^2}+\frac{\alpha_0(D+N)}{1-\beta_1}+\big( \frac{\beta_1}{1-\beta_1} +\frac{1}{2}\big)L\alpha_0^2C_r^2\frac{1}{1-2k}
\end{align}

\begin{lemma}
\label{lemma:mt}
Let $m_t = \beta_1 m_{t-1}+(1-\beta_1)g_t$, let $B_t \in \mathbb{R}^d$, then
\begin{equation}
   \Big \langle B_t, g_t\Big \rangle = \frac{1}{1-\beta_1} \Big(\Big \langle B_t, m_t\Big \rangle -\Big \langle B_{t-1}, m_{t-1}\Big \rangle \Big) +\Big \langle B_{t-1}, m_{t-1}\Big \rangle + \frac{\beta_1}{ 1 - \beta_1}\Big \langle B_{t-1} - B_t, m_{t-1}\Big \rangle
\end{equation}
\end{lemma}
\begin{theorem}
\label{thm:momentum}
Under assumptions 1-4, $\beta_1 < 1, \beta_2 < 1$, also assume $A_{t+1} \leq A_{t}$ element-wise which can be achieved by tracking maximum of $s_t$ as in AMSGrad, $f$ is upper bounded by $F$, $\vert \vert g_t \vert  \vert_{\infty} \leq N$, with learning rate schedule as
\begin{equation}
\label{append_append_eq:lr_schedule}
\alpha_t = \alpha_0 t^{-\eta},\ \ \alpha_0 \leq \frac{C_l}{LC_u^2},\ \  \eta \in (0.5, 1]
\end{equation}
the sequence is generated by 
\begin{equation}
    x_{t+1} = x_t - \alpha_t A_t m_t
\end{equation}
then we have
\begin{equation}
    \frac{1}{T} \sum_{t=1}^T\Big \vert \Big \vert \nabla f(x_t)\Big \vert \Big \vert^2 \leq \frac{1}{\alpha_0 C_l} T^{\eta -1} \Big[ M_f - f^* + E M_g^2 \Big]
\end{equation}
where 
\begin{equation}
    E=  \frac{\beta_1^2}{4L (1-\beta_1)^2} + \frac{1}{1-\beta_1} \alpha_0 M_g + \Big( \frac{\beta_1}{1-\beta_1}+\frac{1}{2} \Big) L\alpha_0^2C_u^2 \frac{1}{1-2\eta}
\end{equation}
\end{theorem}
\begin{proof}
Let $Q_t = \alpha_t A_t \nabla f(x_t)$ and let $Q_0 = Q_1$, we have
\begin{align}
    \sum_{t=1}^T\Big \langle Q_t, g_{h,t}\Big \rangle &= \frac{1}{1-\beta_1}\Big \langle Q_T, m_{h,T}\Big \rangle + \sum_{t=1}^{T}\Big \langle Q_{t-1}, m_{h,t-1}\Big \rangle + \frac{\beta_1}{1-\beta_1} \sum_{t=1}^T\Big \langle Q_{t-1}-Q_t, m_{h,t-1}\Big \rangle \\
    &= \frac{\beta_1}{1-\beta_1}\Big \langle Q_T, m_{h,T}\Big \rangle + \sum_{t=1}^{T}\Big \langle Q_{t}, m_{h,t}\Big \rangle + \frac{\beta_1}{1-\beta_1} \sum_{t=0}^{T-1}\Big \langle Q_{t}-Q_{t+1}, m_{h,t}\Big \rangle \label{eq:sum_prod}
\end{align}
First we derive a lower bound for Eq.~\eqref{eq:sum_prod}.
\begin{align}
   \Big \langle Q_t, g_{h,t}\Big \rangle &=\Big \langle \alpha_t A_t \nabla f(x_t), g_{h,t}\Big \rangle \\
    &=\Big \langle \alpha_{t-1} A_{t-1} \nabla f(x_t), g_{h,t}\Big \rangle -\Big \langle (\alpha_{t-1} A_{t-1}-\alpha_t A_t ) \nabla f(x_t), g_{h,t}\Big \rangle \\
    &\geq\Big \langle \alpha_{t-1} A_{t-1} \nabla f(x_t), g_{h,t}\Big \rangle -\Big \vert \Big \vert \nabla f(x_t) \Big \vert \Big  \vert_{\infty}\Big \vert \Big \vert \alpha_{t-1} A_{t-1} - \alpha_t A_t\Big \vert \Big \vert_1\Big \vert \Big \vert g_{h,t}\Big \vert \Big \vert_{\infty} \\
    &\Big( \textit{By H\"older's inequality} \Big) \nonumber \\
    &\geq \Big \langle \alpha_{t-1} A_{t-1} \nabla f(x_t), g_{h,t}\Big \rangle - \Big(N^2+P^2\Big) \Big(\Big \vert \Big \vert \alpha_{t-1}A_{t-1}\Big \vert \Big \vert_1 -\Big \vert \Big \vert \alpha_t A_t\Big \vert \Big \vert_1 \Big) \\
    &\Big( \textit{Since }\Big \vert \Big \vert \nabla f(x)\Big \vert \Big \vert_{\infty} \leq N,\Big \Vert g_{h,t} \Big \Vert_\infty \leq N+P, \alpha_{t-1} \geq \alpha_t > 0, A_{t-1} \geq A_t >0 \textit{ element-wise} \Big)
\end{align}
Perform telescope sum, we have
\begin{equation}
\label{eq:prod_lower_bound}
    \sum_{t=1}^T\Big \langle Q_t, g_{h,t}\Big \rangle \geq \sum_{t=1}^T\Big \langle \alpha_{t-1} A_{t-1} \nabla f(x_t), g_{h,t}\Big \rangle - \Big(N^2+NP\Big) \Big(\Big \vert \Big \vert \alpha_0 A_0\Big \vert \Big \vert_1 -\Big \vert \Big \vert \alpha_T A_t\Big \vert \Big \vert_1 \Big)
\end{equation}
Next, we derive an upper bound for $\sum_{t=1}^T\Big \langle A_t, g_t\Big \rangle$ by deriving an upper-bound for the RHS of Eq.~\eqref{eq:sum_prod}. We derive an upper bound for each part.

\begin{align}
\langle Q_t, m_{h,t}\Big \rangle &=\Big \langle \alpha_t A_t \nabla f(x_t), m_{h,t}\Big \rangle =\Big \langle \nabla f(x_t),  \alpha_t A_t m_{h,t}\Big \rangle \\
&=\Big \langle \nabla f(x_t), x_{t} - x_{t+1}\Big \rangle \\
&\leq f(x_t) - f(x_{t+1}) + \frac{L}{2}\Big \vert \Big \vert x_{t+1} - x_t\Big \vert \Big \vert^2 \Big( \textit{By L-smoothness of }f \Big) \label{eq:bd13}
\end{align}
Perform telescope sum, we have
\begin{align}
\sum_{t=1}^T\Big \langle Q_t, m_{h,t}\Big \rangle \leq f(x_1) - f(x_{T+1}) + \frac{L}{2} \sum_{t=1}^T\Big \vert \Big \vert \alpha_t A_t m_{h,t}\Big \vert \Big \vert^2 
\label{eq:bd14}
\end{align}
\begin{align}
\langle Q_t - Q_{t+1}, m_{h,t}\Big \rangle &=\Big \langle \alpha_t A_t \nabla f(x_t) - \alpha_{t+1} A_{t+1} \nabla f(x_{t+1}), m_{h,t}\Big \rangle \\
&=\Big \langle \alpha_t A_t \nabla f(x_t) - \alpha_t A_t \nabla f(x_{t+1})  , m_{h,t}\rangle \nonumber \\
&+\Big \langle \alpha_t A_t \nabla f(x_{t+1}) - \alpha_{t+1} A_{t+1} \nabla f(x_{t+1}) , m_{h,t}\rangle \\
&=\Big \langle \nabla f(x_t) - \nabla f(x_{t+1}), \alpha_t A_t m_{h,t}\Big \rangle +\Big \langle (\alpha_t A_t - \alpha_{t+1} A_{t+1}) \nabla f(x_t) ,m_{h,t}\Big \rangle \\
&=\Big \langle \nabla f(x_t) - \nabla f(x_{t+1}), x_t - x_{t+1}\Big \rangle +\Big \langle \nabla f(x_t), ( \alpha_t A_t - \alpha_{t+1} A_{t+1}) m_{h,t}\Big \rangle \\
&\leq L\Big \vert \Big \vert x_{t+1} - x_t\Big \vert \Big \vert^2 +\Big \langle \nabla f(x_t), ( \alpha_t A_t - \alpha_{t+1} A_{t+1}) m_{h,t}\Big \rangle \\
&\Big( \textit{By smoothness of }f \Big) \nonumber \\
&\leq L\Big \vert \Big \vert x_{t+1} - x_t\Big \vert \Big \vert^2 +\Big \vert \Big \vert \nabla f(x_t)\Big \vert \Big \vert_\infty\Big \vert \Big \vert \alpha_t A_t - \alpha_{t+1} A_{t+1}\Big \vert \Big \vert_1\Big \vert \Big \vert m_{h,t}\Big \vert \Big \vert_\infty \\
&\Big( \textit{By H\"older's inequality} \Big) \nonumber \\
&\leq L\Big \vert \Big \vert x_{t+1} - x_t\Big \vert \Big \vert^2 + \Big(N^2+NP\Big) \Big(\Big \vert \Big \vert \alpha_t A_t\Big \vert \Big \vert_1 -\Big \vert \Big \vert \alpha_{t+1} A_{t+1}\Big \vert \Big \vert_1 \Big) \\
&\Big( \textit{Since } \alpha_t \geq \alpha_{t+1} \geq 0, A_t \geq A_{t+1} \geq 0, \textit{element-wise} \Big)
\label{eq:bd15}
\end{align}
Perform telescope sum, we have 
\begin{align}
    \sum_{t=1}^{T-1}\Big \langle Q_t - Q_{t+1}, m_{h,t}\rangle \leq L \sum_{t=1}^{T-1}\Big \vert \Big \vert \alpha_t A_t m_{h,t}\Big \vert \Big \vert^2 + \Big(N^2+NP\Big) \Big(\Big \vert \Big \vert \alpha_1 A_1\Big \vert \Big \vert_1 -\Big \vert \Big \vert \alpha_T A_T\Big \vert \Big \vert_1 \Big)
    \label{eq:bd16}
\end{align}
We also have
\begin{align}
   \Big \langle Q_T, m_{h,T}\Big \rangle &=\Big \langle \alpha_T A_T \nabla f(x_T), m_{h,T} \Big \rangle =\Big \langle \nabla f(x_T), \alpha_T A_T m_{h,T}\Big \rangle \\
    &\leq L \frac{1-\beta_1}{\beta_1}\Big \vert \Big \vert  \alpha_T A_T m_{h,T}\Big \vert \Big \vert^2 + \frac{\beta_1}{4L(1-\beta_1)}\Big \vert \Big \vert \nabla f(x_T)\Big \vert \Big \vert^2 \\
    &\Big( \textit{By Young's inequality} \Big) \nonumber \\
    &\leq L \frac{1-\beta_1}{\beta_1}\Big \vert \Big \vert  \alpha_T A_T m_{h,T}\Big \vert \Big \vert^2 + \frac{\beta_1d}{4L (1-\beta_1)} N^2
    \label{eq:bd17}
\end{align}
Combine Eq.~\eqref{eq:bd14}, Eq.~\eqref{eq:bd16} and Eq.~\eqref{eq:bd17} into Eq.~\eqref{eq:sum_prod}, we have
\begin{align}
    \sum_{t=1}^T\Big \langle Q_t, g_{h,t}\Big \rangle &\leq L\Big \vert \Big \vert  \alpha_T A_T m_{h,T}\Big \vert \Big \vert^2 + \frac{\beta_1^2d}{4L (1-\beta_1)^2} N^2 \nonumber \\ 
    &+ f(x_1) - f(x_{T+1}) + \frac{L}{2} \sum_{t=1}^T\Big \vert \Big \vert \alpha_t A_t m_{h,t}\Big \vert \Big \vert^2 \nonumber \\
    &+ \frac{\beta_1}{1-\beta_1} L \sum_{t=1}^{T-1}\Big \vert \Big \vert \alpha_t A_t m_{h,t}\Big \vert \Big \vert^2 + \frac{\beta_1}{1-\beta_1} \Big(N^2+NP\Big) \Big(\Big \vert \Big \vert \alpha_1 A_1\Big \vert \Big \vert_1 -\Big \vert \Big \vert \alpha_T A_T\Big \vert \Big \vert_1 \Big) \\
    &\leq f(x_1) - f(x_{T+1}) + \Big( \frac{\beta_1}{1-\beta_1}+\frac{1}{2} \Big)L \sum_{t=1}^T\Big \vert \Big \vert \alpha_t A_t m_{h,t}\Big \vert \Big \vert^2 \nonumber \\
    &+ \frac{\beta_1^2d}{4L (1-\beta_1)^2}N^2 + \frac{\beta_1}{1-\beta_1}\Big(N^2+NP\Big) \Big \vert \Big \vert \alpha_1 A_1\Big \vert \Big \vert_1
    \label{eq:bd18}
\end{align}
Combine Eq.~\eqref{eq:prod_lower_bound} and Eq.~\eqref{eq:bd18}, we have
\begin{align}
 \sum_{t=1}^T\Big \langle \alpha_{t-1} A_{t-1} \nabla f(x_t), g_{h,t}\Big \rangle &- \Big(N^2+NP\Big) \Big(\Big \vert \Big \vert \alpha_0 A_0\Big \vert \Big \vert_1 -\Big \vert \Big \vert \alpha_T A_T\Big \vert \Big \vert_1 \Big) \leq    \sum_{t=1}^T\Big \langle Q_t, g_{h,t}\Big \rangle \nonumber \\
 &\leq f(x_1) - f(x_{T+1}) + \Big( \frac{\beta_1}{1-\beta_1}+\frac{1}{2} \Big)L \sum_{t=1}^T\Big \vert \Big \vert \alpha_t A_t m_{h,t}\Big \vert \Big \vert^2 \nonumber \\
    &+ \frac{\beta_1^2d}{4L (1-\beta_1)^2}N^2 + \frac{\beta_1}{1-\beta_1}\Big(N^2+NP\Big) \Big \vert \Big \vert \alpha_1 A_1\Big \vert \Big \vert_1
\end{align}
Hence we have
\begin{align}
    &\sum_{t=1}^T\Big \langle \alpha_{t-1} A_{t-1} \nabla f(x_t), g_{h,t}\Big \rangle \leq f(x_1) - f(x_{T+1}) + \Big( \frac{\beta_1}{1-\beta_1} + \frac{1}{2} \Big)L \sum_{t=1}^T\Big \vert \Big \vert \alpha_t A_t m_t\Big \vert \Big \vert^2 \nonumber \\
    &+ \frac{\beta_1^2d}{4L (1-\beta_1)^2}N^2 + \Big(\Big \vert \Big \vert \alpha_0 A_0\Big \vert \Big \vert_1+\frac{\beta_1}{1-\beta_1} \Big \vert \Big \vert \alpha_1 A_1\Big \vert \Big \vert_1\Big) \Big(N^2+NP\Big)\\
    &\leq f(x_1) - f^* + \Big( \frac{\beta_1}{1-\beta_1}+\frac{1}{2} \Big) L \alpha_0^2 \Big(N+P\Big)^2 C_r^2d \sum_{t=1}^T t^{-2k} \nonumber \\
    &+ \frac{\beta_1^2d}{4L (1-\beta_1)^2}N^2 + \Big(\Big \vert \Big \vert \alpha_0 A_0\Big \vert \Big \vert_1+\frac{\beta_1}{1-\beta_1} \Big \vert \Big \vert \alpha_1 A_1\Big \vert \Big \vert_1\Big) \Big(N^2+NP\Big)\\
    &\leq f(x_1) - f^* \nonumber \\
    &+ \Big(N+P\Big)^2  \Big[ \frac{\beta_1^2}{4L (1-\beta_1)^2} +  \Big \vert \Big \vert \alpha_0 H_0 \Big \vert \Big \vert_1 + \frac{\beta_1}{1-\beta_1} \Big \vert \Big \vert \alpha_1 H_1 \Big \vert \Big \vert_1 + \Big( \frac{\beta_1}{1-\beta_1}+\frac{1}{2} \Big) L\alpha_0^2C_u^2 \Big(1+ \int_{1}^{T} t^{-2p}\, \mathrm{d}t\Big) \Big] \\
    &\leq f(x_1)-f^* + \Big(N+P\Big)^2 \underbrace{ \Big[  \frac{\beta_1^2}{4L (1-\beta_1)^2} + \frac{1}{1-\beta_1} \alpha_0 C_r + \Big( \frac{\beta_1}{1-\beta_1}+\frac{1}{2} \Big) L\alpha_0^2C_u^2 \Big(1+ \int_{1}^{T} t^{-2p}\, \mathrm{d}t\Big) \Big]}_{J}
\end{align}
Take expectations on both sides, we have
\begin{align}
    \sum_{t=1}^T\Big \langle \alpha_{t-1} A_{t-1} \nabla f(x_t), \nabla f(x_t)+H_t\Big \rangle \leq \mathbb{E} f(x_1) - f^*+J \Big(N+P\Big)^2 \leq F - f^* + J \Big(N+P\Big)^2
    \label{eq:bd19}
\end{align}
Note that we have $\alpha_t$ decays monotonically with $t$, hence
\begin{align}
    \sum_{t=1}^T\Big \langle \alpha_{t-1} A_{t-1} \nabla f(x_t), \nabla f(x_t)+H_t\Big \rangle &\geq \alpha_0 T^{-k} \sum_{t=1}^T\Big \langle A_{t-1} \nabla f(x_t), \nabla f(x_t)+H_t\Big \rangle \\
    &\geq \alpha_0 T^{1-k} C_l \frac{1}{T}\sum_{t=1}^T\Big[\Big \vert \Big \vert \nabla f(x_t)\Big \vert \Big \vert^2+\Big \langle \nabla f(x_t), H_t\Big \rangle \Big]\\
    &\geq \alpha_0 T^{1-k} C_l \frac{1}{T}\sum_{t=1}^T\Big[\Big \vert \Big \vert \nabla f(x_t)\Big \vert \Big \vert^2-\Big \vert \Big \vert \nabla f(x_t)\Big \vert \Big \vert \Big \vert \Big \vert H_t\Big \vert \Big \vert  \Big]\\
    &\geq \alpha_0 T^{1-k} C_l \frac{1}{T}\sum_{t=1}^T\Big[\Big \vert \Big \vert \nabla f(x_t)\Big \vert \Big \vert^2-dNP_t\Big]
    \label{eq:bd20}
\end{align}
Combine Eq.~\eqref{eq:bd19} and Eq.~\eqref{eq:bd20}, assume $f$ is upper bounded by $M_f$, we have
\begin{align}
    \frac{1}{T} \sum_{t=1}^T\Big \vert \Big \vert \nabla f(x_t)\Big \vert \Big \vert^2 &\leq \frac{1}{\alpha_0 C_l} T^{k -1} \Big[ F - f^* + J\Big(N+P_1\Big)^2 \Big]+\frac{1}{T} \sum_{t=1}^T dNP_t\\
&= \frac{1}{\alpha_0 C_l} T^{k -1} \Big[ F - f^* + J\Big(N+\frac{1}{1-\beta_1}C_l^{-1}C_rN\Big)^2 \Big]\nonumber \\
&+\frac{1}{T} \Big[\sum_{t=2}^T dN\frac{2(t-1)^{-k}}{1-\beta_1}C_l^{-1}C_rN+\frac{1}{1-\beta_1}dC_l^{-1}C_rN^2\Big]\\
&\leq \frac{1}{\alpha_0 C_l} T^{k -1} \Big[ F - f^* + JN^2\Big(1+\frac{1}{1-\beta_1}C_l^{-1}C_r\Big)^2 \Big]\nonumber \\
&+\frac{1}{T} \Big[\frac{2(T-1)^{1-k}-2k}{(1-\beta_1)(1-k)}dC_l^{-1}C_rN^2+\frac{1}{1-\beta_1}dC_l^{-1}C_rN^2\Big]\\
&\leq \frac{1}{\alpha_0 C_l} T^{k -1} \Big[ F - f^* + JN^2\Big(1+\frac{1}{1-\beta_1}C_l^{-1}C_r\Big)^2 \Big]\nonumber \\
&+\frac{1}{T} \Big[\frac{2T^{1-k}}{(1-\beta_1)(1-k)}dC_l^{-1}C_rN^2+\frac{T^{1-k}}{(1-\beta_1)(1-k)}dC_l^{-1}C_rN^2\Big]\\
&\leq \frac{1}{\alpha_0 C_l} T^{k -1} \Big[ F - f^* + JN^2\Big(1+\frac{1}{1-\beta_1}C_l^{-1}C_r\Big)^2 \Big]\nonumber \\
&+3T^{-k}\frac{dC_l^{-1}C_rN^2}{(1-\beta_1)(1-k)}
\end{align}
\end{proof}

\section{Validation on Simple but Representative Functions}
In this section, we validate the performance of HVAdam for the functions in Figure 2 in paper. In every case, HVAdam chooses the optimal direction and gets the lowest loss in the end.\\

\textbf{Function 1}
\label{functions}
The loss function is $f_1(x,y)=4\sqrt{2}\vert x\vert+\frac{\sqrt{2}}{10}\vert y\vert$. From the optimization trajectories, we can find that the adaptive optimizers update faster than SGDM in the direction of $-\vec{j}$. After 1500 epochs, we can see that HVAdam achieves the lowest loss. In addition, the losses of AdaBelief and Adam are lower than the loss of SGDM. The result is consistent with the previous analysis. 

\textbf{Function 2},
The loss function is $f_2(x,y) = f_2(x,y)=4\vert x-y\vert+\vert\frac{x+y}{10}\vert$. From the optimization trajectories, we can find that HVAdam still updates faster than others in the direction of $-\vec{i}-\vec{j}$, while the other adaptive optimizers have no advantage compared to SGDM. After 1500 epochs, the losses of other adaptive optimizers are very close to the loss of SGDM, while the loss of HVAdam is much lower. The result is also consistent with the previous analysis and shows the effectiveness of our method.  

\textbf{Function 3},
The loss function is $f_3(x,y)=\ln(1+\text{Beale}(x,y))/10$.
Beale\cite{beale1955minimizing} is a commonly used function to test optimizer performance.
And we have further processed it according to \cite{zhuang2020adabelief}.
We use $f_3$ to show that $v_t$ can provide us with intuitive information that helps us to find a better updated direction, even if there is no direction onto which the projection of every step's gradient is not changed. From the optimization trajectories, we can find that HVAdam chooses the proper direction to update and still reaches the optimal point fastest.

\textbf{Function 4},
The loss function is $f_4(x,y)=\text{Rosenbrock}(x,y)$. Rosenbrock\cite{rosenbrock1960automatic} is also a commonly used function to test optimizer performance. Our optimizer also performs better than other optimizers, showing its advantages in a common situation.  

\paragraph{Case study of function in deep learning}
\label{sec:deeplr}
We validate HVAdam for the simple cases above and believe that these cases also exist broadly in deep learning. For example, due to the interaction between neurons, most networks behave following the form $(x\pm y)$\cite{zhuang2020adabelief}. And the absolute value function in $f_1$ and $f_2$ is similar to the ReLU activation \cite{glorot2011deep}. Additionally, the parameters' partial derivatives in deep learning usually differ greatly, implying that directional derivatives may differ greatly, as demonstrated by the first two representative loss functions above. In summary, since HVAdam outperforms the other optimizers in the four loss functions, and these functions largely cover the terms in generic loss functions, we expect HVAdam to perform well in most deep-learning problems.

\section{Details of experiments}
\subsection{Conﬁguration of optimizers}
In this section, we provide a detailed description of the hyperparameters used by different optimizers on various tasks. Our experiments are run with the Pytroch deep learning platform in an environment comprising an Intel Xeon Gold 5117 CPU, 128-GB RAM, and 24-GB NVIDIA GTX 3090 graphics processing unit.

For optimizers other than HVAdam, we adopt the recommended parameters for the identical experimental setup as indicated in the literature of AdaBelief\cite{zhuang2020adabelief}, Adai\cite{xie2022adaptive} and Lookaround\cite{zhang2023Lookaround}.

For HVAdam, the searching scheme of hyperparameter settings is concluded as follows: 
\begin{itemize}
    \item CNNs on CIFAR-10 and CIFAR-100\cite{krizhevsky2009learning}: we search learning rate from \{0.1,0.01,0.001,0.0001\}, search $\gamma$ from \{0,1,2,...,10\}, $\epsilon$ is set fixed $10^{-8}$, $\beta_1=0.9$, $\beta_2=0.999$.
    \item ViT on CIFAR-10 and CIFAR-100: we search the learning rate from \{0.001,0.0001,0.00001\}, $\gamma$ is set fixed 1, search $\epsilon$ from \{$10^{-12}$,$10^{-16}$\}, $\beta_1=0.9$, $\beta_2=0.999$.
    \item ImageNet\cite{deng2009imagenet}: we search the learning rate from \{0.004,0.001\}, $\gamma$ is set to 0.5, $\epsilon$ is set fixed $10^{-8}$, $\beta_1=0.9$, $\beta_2=0.999$.
    \item LSTM: we search learning rate from \{0.01,0.001,0.0001\}, $\gamma$ is set fixed 0.5.
    \item WGAN:\cite{deng2009imagenet}: we search the learning rate from \{0.01,0.001\}, search $\gamma$ from \{0,0.1,0.2,0.5\}, $\epsilon$ is set fixed $10^{-12}$, $\beta_1=0.5$, $\beta_2=0.999$.
    \item WGAN-GP:\cite{deng2009imagenet}: we search the learning rate from \{0.01,0.001\}, search $\gamma$ from \{0,0.1,0.2,0.5\}, $\epsilon$ is set fixed $10^{-12}$, $\beta_1=0.5$, $\beta_2=0.999$.
    \item SNGAN\cite{deng2009imagenet}: we search for the learning rate from \{0.01,0.001\}, search $\gamma$ from \{0,0.001,0.002,0.1,0.5\}, $\epsilon$ is set to$10^{-12}$, $\beta_1=0.5$, $\beta_2=0.999$.
\end{itemize}

\subsection{Conﬁguration of numerical experiments}
We set $lr=0.001$, $\beta_1=0.9$, $\beta_2=0.999$, $\epsilon=10^{-12}$ for Adam and AdaBelief.
We set $lr=0.001$, $\beta=0.9$ for SGDM.
We set $lr=0.001$, $\beta_1=0.9$, $\beta_2=0.999$, $\epsilon=10^{-12}$, $\gamma=0.5$ for HVAdam. And we choose $lr(\delta_{t_2},\widehat{\delta_{t_2}}):=\left\{
        \begin{array}{ll}
            \frac{\delta_{t_2}}{0.9}\cdot10 ,& \text{if }\widehat{\delta_{t_2}} \ge 0.1 \\
            0,   & \text{otherwise}
        \end{array}
    \right.
$.\\
The datasets used in the experiments are shown in Table.~\ref{table:datasets}.
\begin{table}[b]
\caption{The datasets used in experiments.}
\label{table:datasets}
\scalebox{0.73}{
\begin{tabular}{c|c|c|c|c|c|c|c|c|c}
\hline
Model               & VGG11               & ResNet34          & DenseNet121        & LSTM           & WGAN        & WGAN-GP           & SN-GAN         & ViT-B/16  &ResNet50                        \\ \hline
Dataset & CIFAR & CIFAR & CIFAR &  Penn TreeBank & CIFAR-10 & CIFAR-10 & CIFAR-10  &  CIFAR &ImageNet \\ \hline
\end{tabular}
}
\end{table}
Our model parameter settings are the same as those listed in the references in the main text. 
Other experiments are based on the official code for AdaBound, AdaBelief and Lookaround.\footnote{https://github.com/juntang-zhuang/Adabelief-Optimizer}\footnote{https://github.com/Ardcy/Lookaround/tree/main} The training and test data of image classification and natural language processing is shown in Fig.~\ref{fig:cifar} and Fig.~\ref{fig:LSTM}. And the images generated by GANs are shown in Fig.~\ref{fig:wgan} and Fig.~\ref{fig:wgangp}.

\begin{figure}
\begin{subfigure}[b]{0.33\textwidth}
\includegraphics[width=\linewidth]{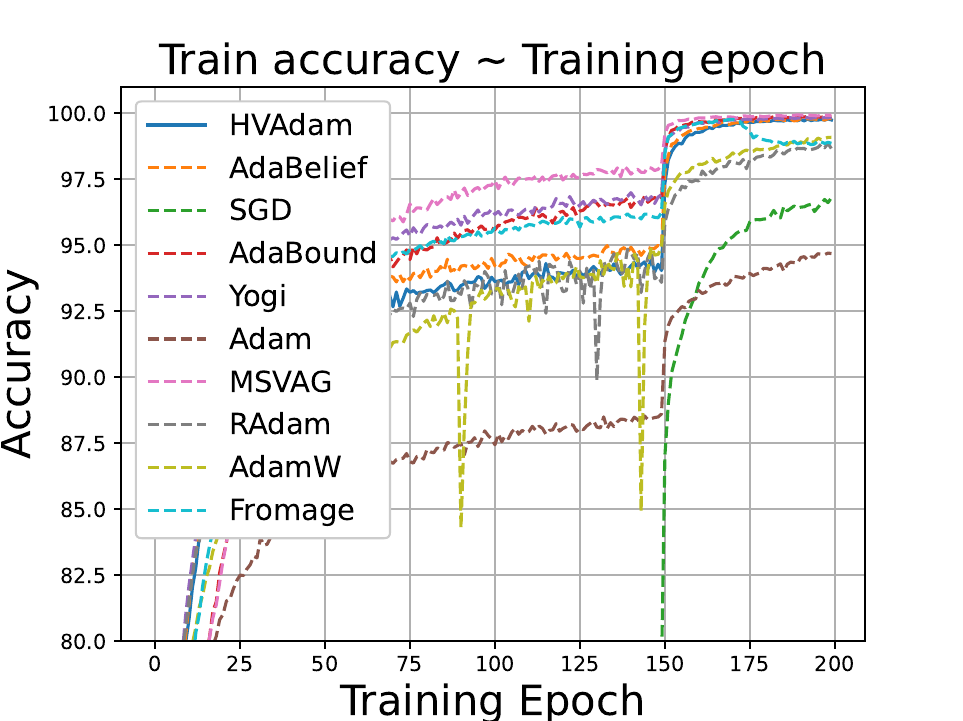}
\caption{\small{
VGG11
}}
\end{subfigure}
\begin{subfigure}[b]{0.32\textwidth}
\includegraphics[width=\linewidth]{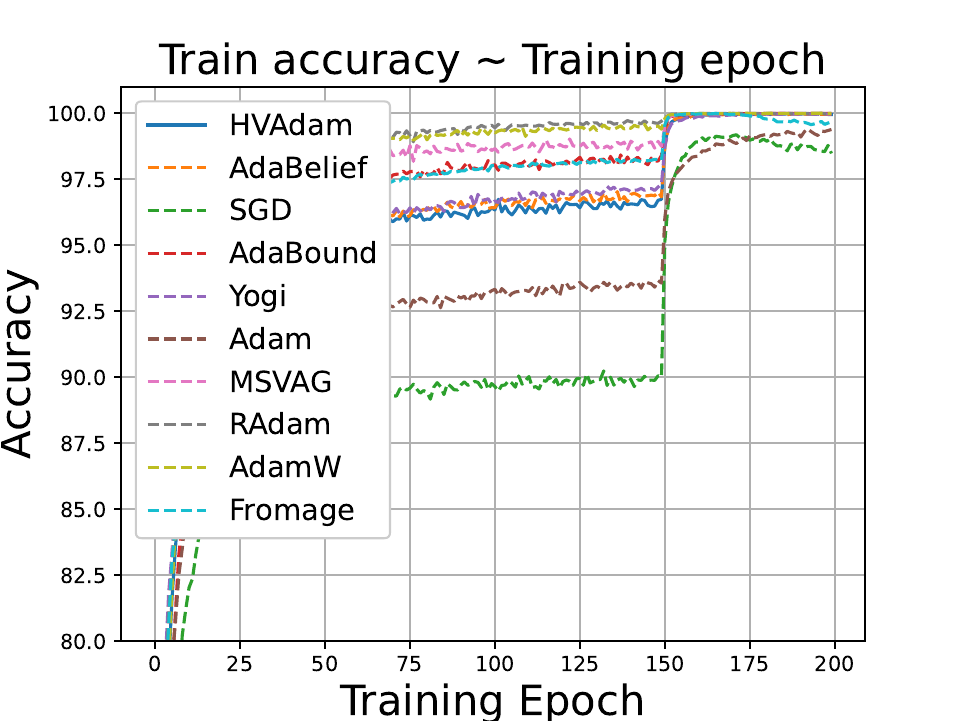}
\caption{\small{
ResNet34 
}}
\end{subfigure}
\begin{subfigure}[b]{0.32\textwidth}
\includegraphics[width=\linewidth]{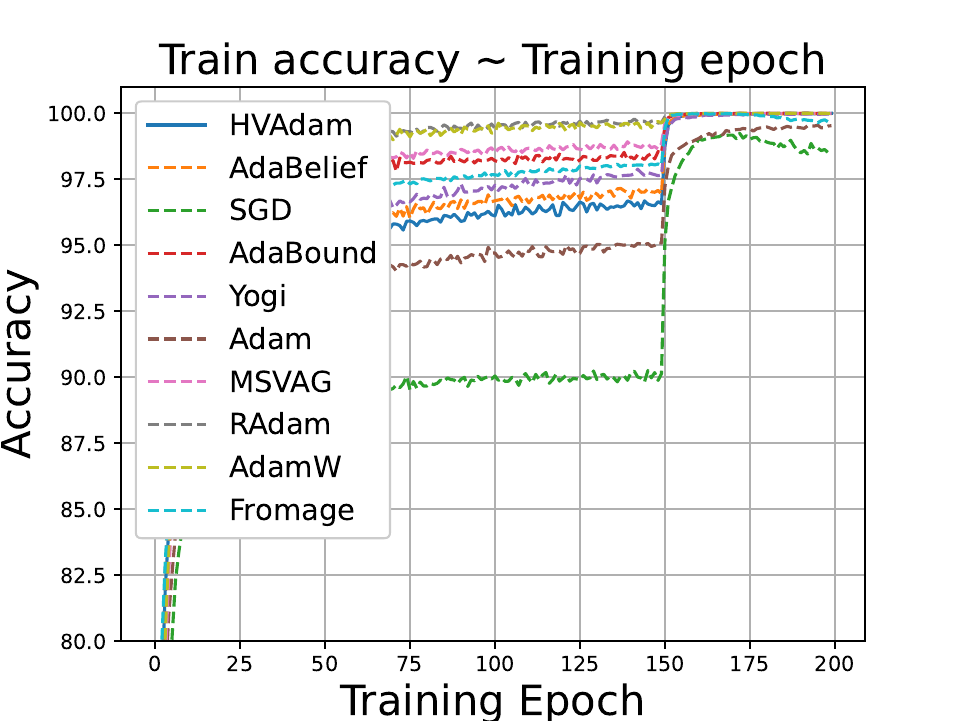}
\caption{\small{
DenseNet121
}}
\end{subfigure}
\\
\begin{subfigure}[b]{0.33\textwidth}
\includegraphics[width=\linewidth]{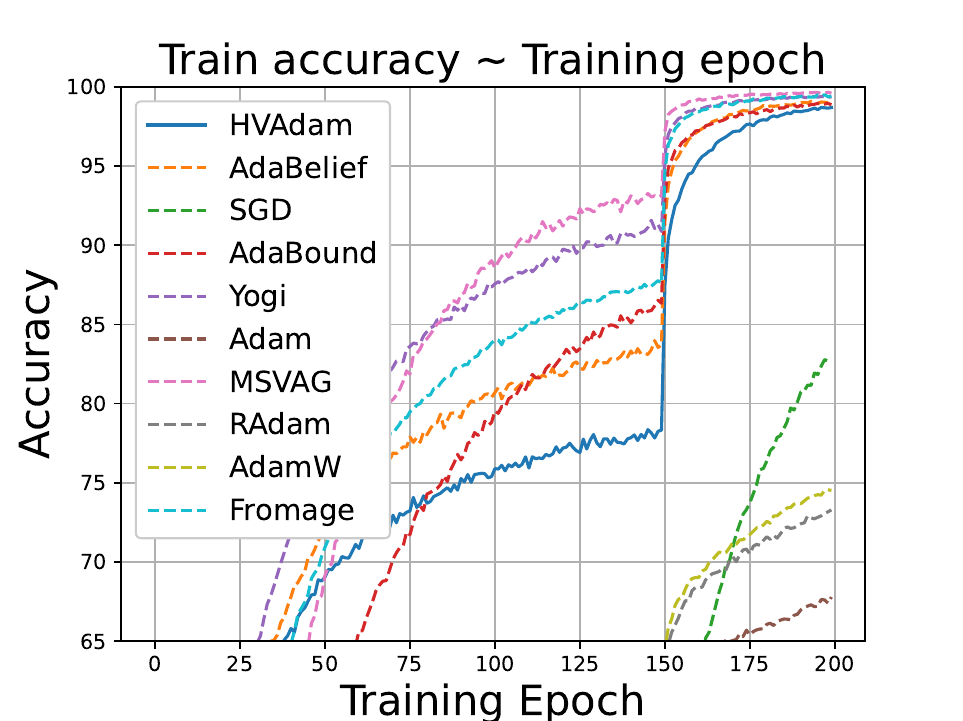}
\caption{\small{
VGG11
}}
\end{subfigure}
\begin{subfigure}[b]{0.32\textwidth}
\includegraphics[width=\linewidth]{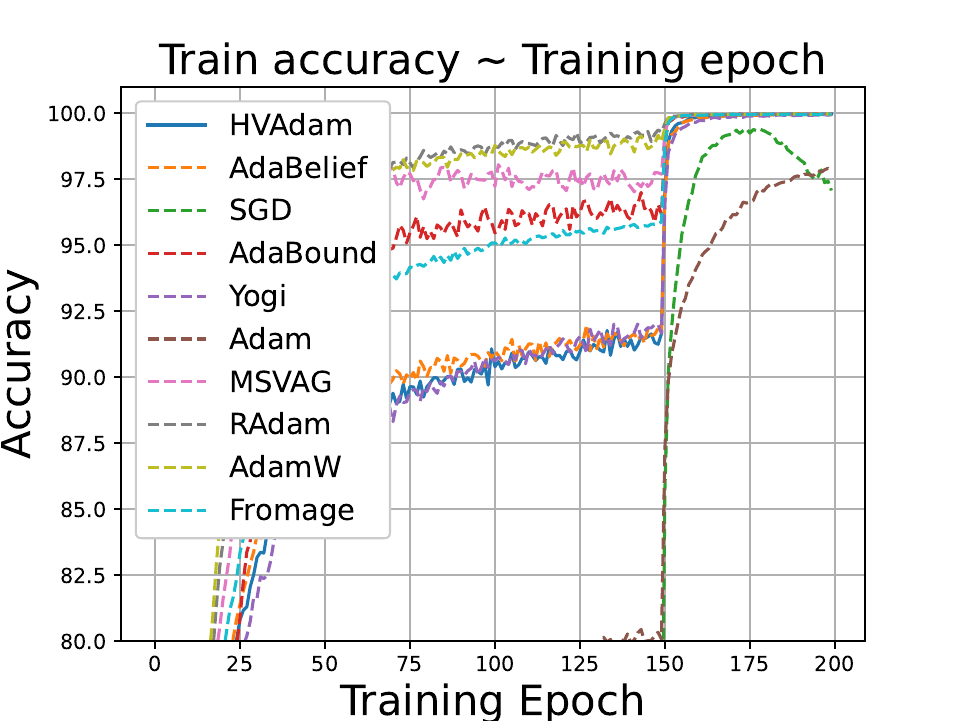}
\caption{\small{
ResNet34 
}}
\end{subfigure}
\begin{subfigure}[b]{0.32\textwidth}
\includegraphics[width=\linewidth]{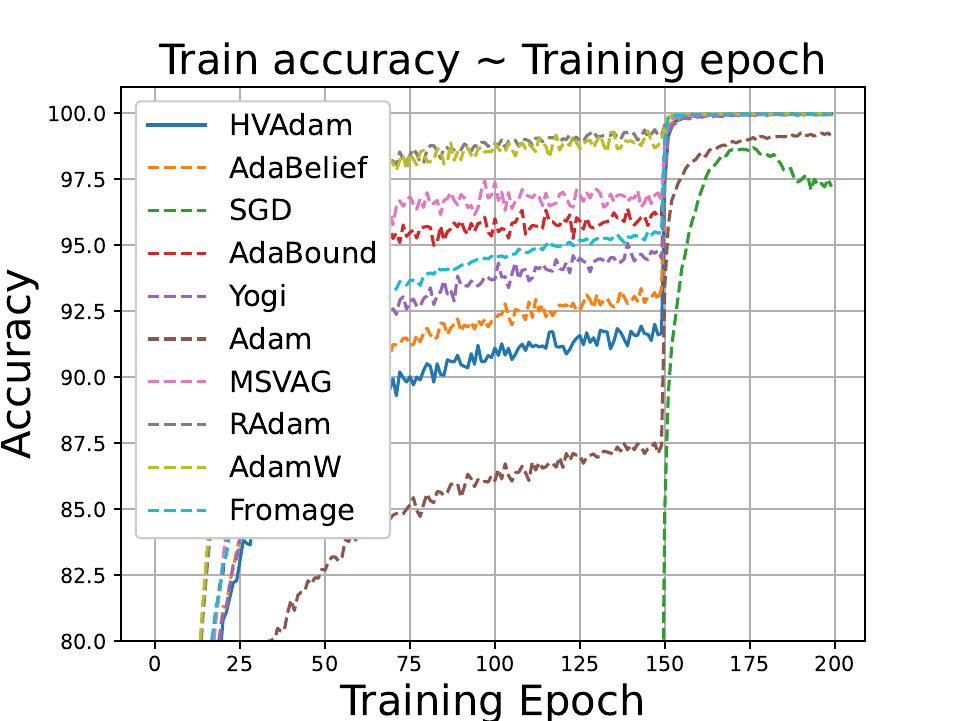}
\caption{\small{
DenseNet121 
}}
\end{subfigure}
\caption{Training accuracies with three models using different optimizers on CIFAR-10 and CIFAR-100.Top: training accuracies of different Network Models on CIFAR-10. Bottom: training accuracies of different Network Models on CIFAR-100}
\label{fig:cifar}
\end{figure}

\begin{figure}
    \begin{subfigure}[b]{0.33\textwidth}
    \includegraphics[width=\linewidth]{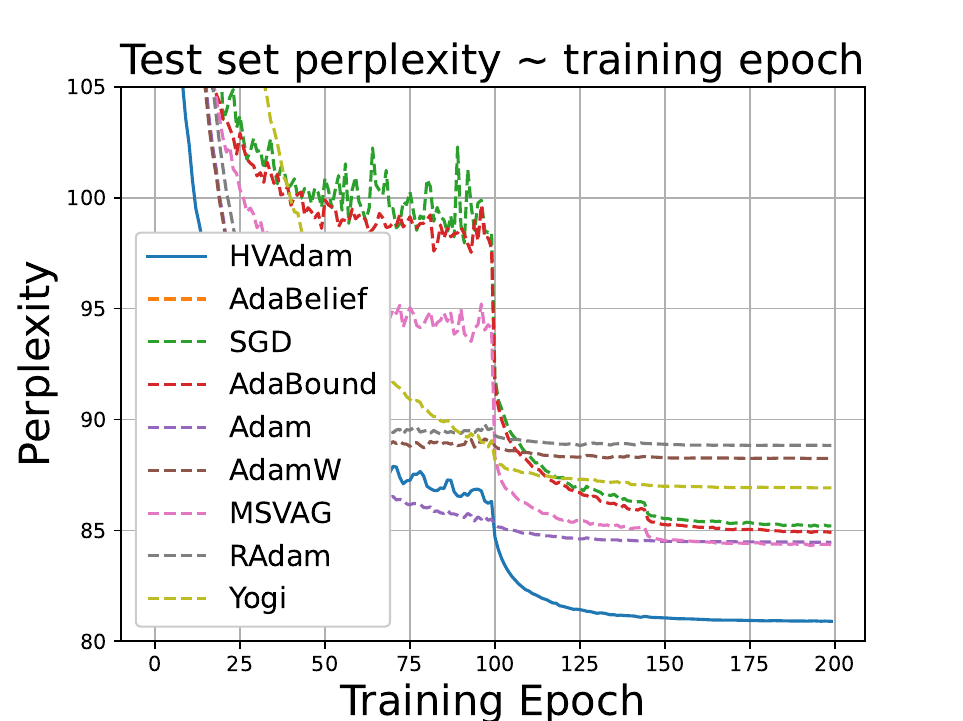}
    \end{subfigure}
    \begin{subfigure}[b]{0.32\textwidth}
    \includegraphics[width=\linewidth]{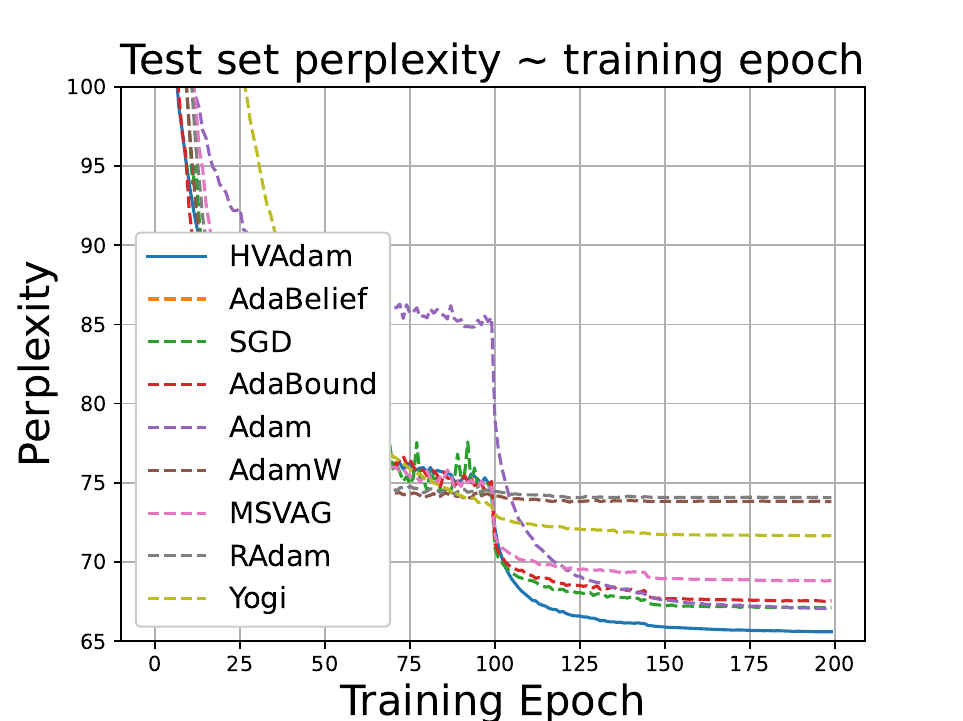}
    \end{subfigure}
    \begin{subfigure}[b]{0.32\textwidth}
    \includegraphics[width=\linewidth]{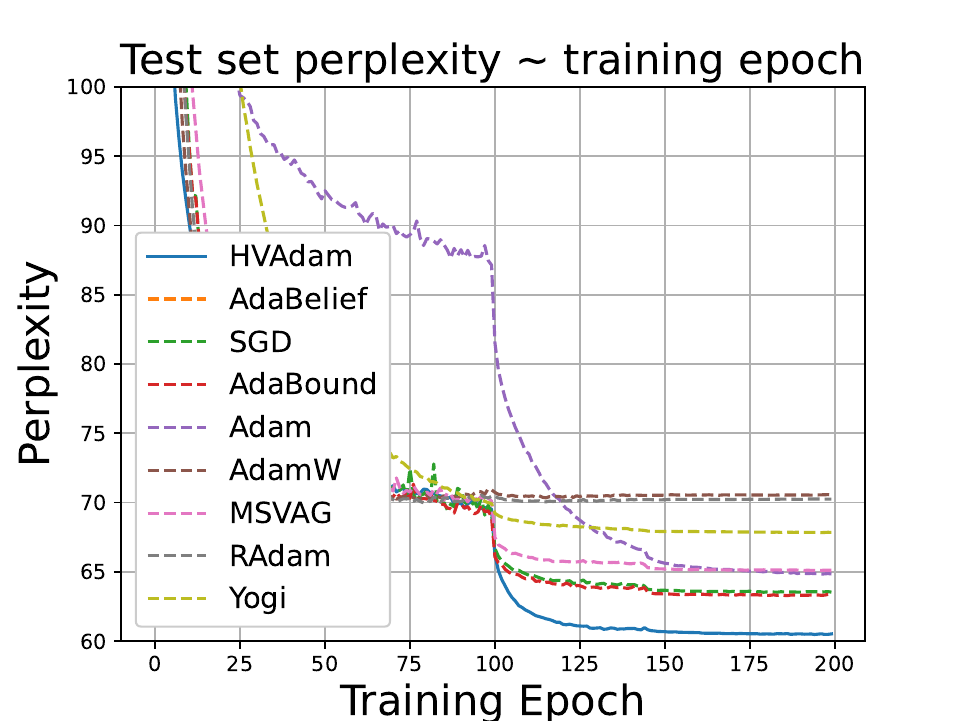}
    \end{subfigure}
    \begin{subfigure}[b]{0.33\textwidth}
    \includegraphics[width=\linewidth]{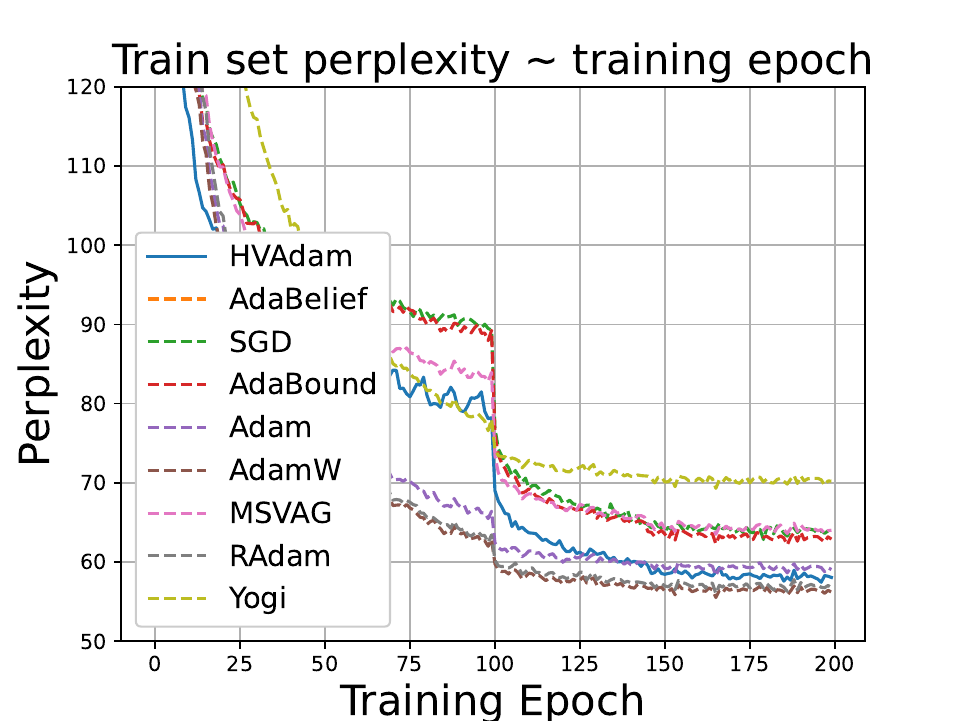}
    \end{subfigure}
    \begin{subfigure}[b]{0.32\textwidth}
    \includegraphics[width=\linewidth]{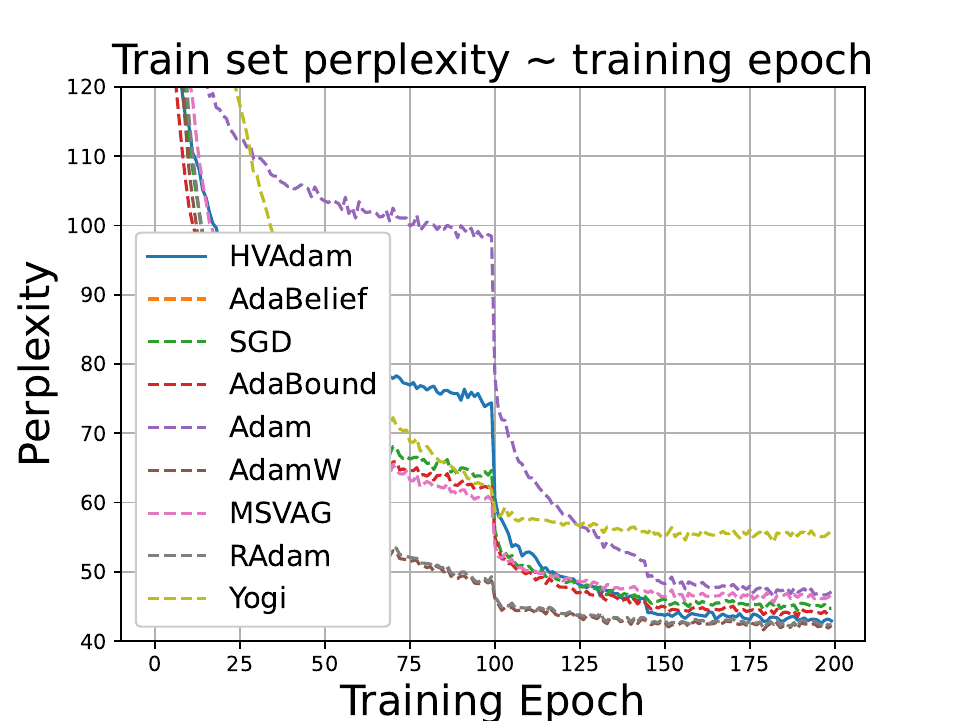}
    \end{subfigure}
    \begin{subfigure}[b]{0.32\textwidth}
    \includegraphics[width=\linewidth]{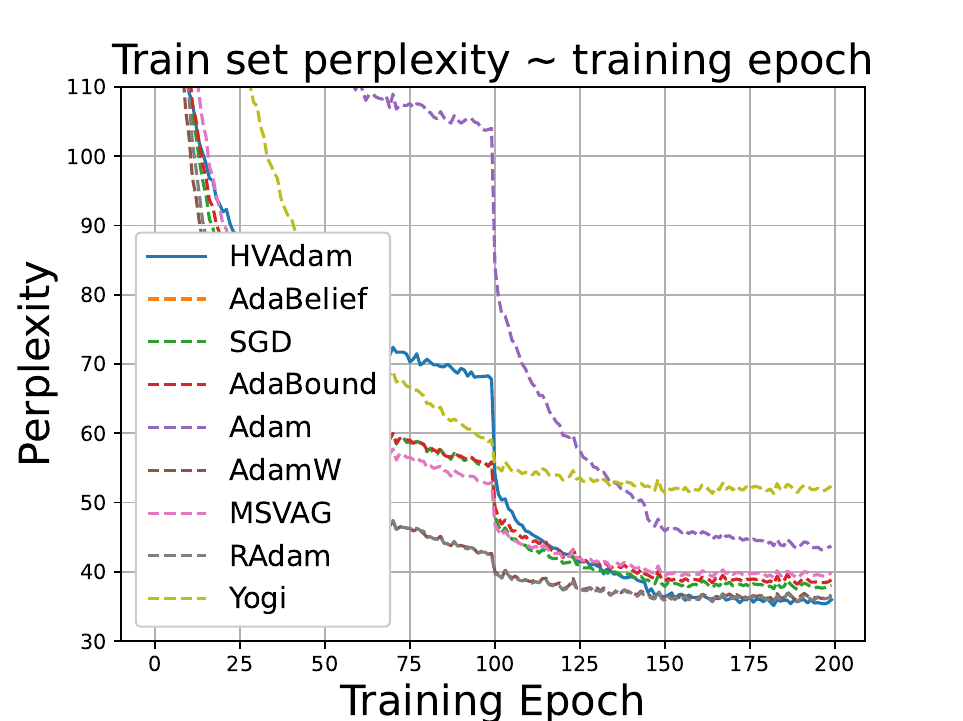}
    \end{subfigure}
    \caption{The test and training perplexity on Penn Treebank for 1,2,3-layer LSTM from left to right. \textbf{Lower} is better.}
     \label{fig:LSTM}
\end{figure}

\begin{figure}
    \centering
    \includegraphics[width=1.0\linewidth]{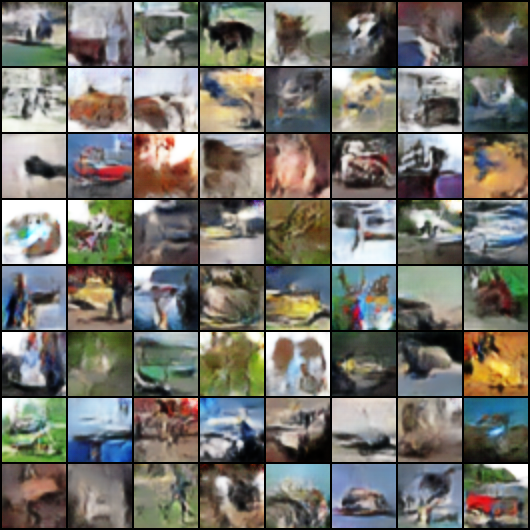}
    \caption{Generated images by the WGAN trained with HVAdam}
    \label{fig:wgan}
\end{figure}

\begin{figure}
    \centering
    \includegraphics[width=1.0\linewidth]{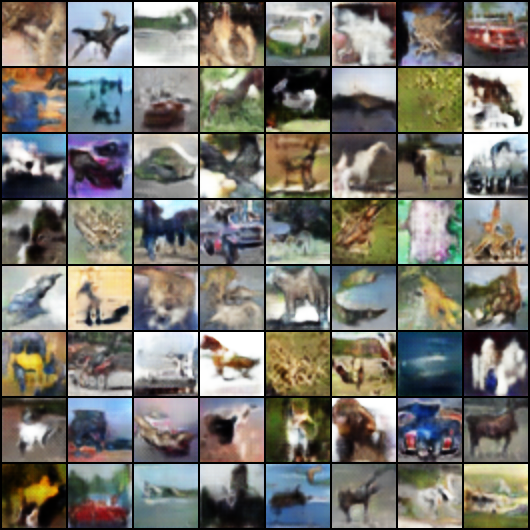}
    \caption{Generated images by the WGAN-GP trained with HVAdam}
    \label{fig:wgangp}
\end{figure}

\section{Additional experiments}
We add additional experiments in this section.
we show the comparison with Padam in Table~\ref{table:padam_compare} . The update of Padam can be written as $\theta_{t+1}=\theta_t - \alpha_1 m_t / a_t^p$. Padam also provides a different adjustment range of $\alpha_1$. It is easy to see when $p=0$, Padam is equivalent to SGD, and when $p=0.5$, Padam is equivalent to Adam. We compare with Padam to validate which adjustment strategy is better since we set $\alpha_2$ as 0. As a result, HVAdam outperforms Padam and demonstrates the advantages of our method.
\begin{table}[h]
\centering
\scalebox{0.73}{
\begin{tabular}{c|c|c|ccccc}
\hline
\multirow{2}{*}{} & \multirow{2}{*}{HVAdam} & \multirow{2}{*}{AdaBelief} & \multicolumn{5}{c}{Padam}                                         \\ \cline{4-8} 
                   &                           & \multicolumn{1}{c|}{} & \multicolumn{1}{c|}{p=1/2 (Adam)}                          & \multicolumn{1}{c|}{p=1/4}                    & \multicolumn{1}{c|}{p=1/8}                          
                 & \multicolumn{1}{c|}{p=1/16}               & p = 0 (SGD)\\ \hline
FID (WGAN)    & \textbf{65.48$\pm$ 1.79} & \multicolumn{1}{c|}{ \underline{82.85$\pm$ 2.21} } & \multicolumn{1}{c|}{106.38$\pm$9.76} & \multicolumn{1}{c|}{422.62$\pm$35.68} & \multicolumn{1}{c|}{330.44$\pm$26.62} & \multicolumn{1}{c|}{357.26$\pm$32.39} & 459.01$\pm$14.62 \\ \hline
FID (WGAN-GP)    & \textbf{58.76$\pm$ 5.29} & \multicolumn{1}{c|}{75.37$\pm$7.37} & \multicolumn{1}{c|}{ \underline{71.87$\pm$ 0.83}}  & \multicolumn{1}{c|}{152.34$\pm$17.49} & \multicolumn{1}{c|}{205.57$\pm$13.79} & \multicolumn{1}{c|}{228.40$\pm$18.24} & 236.99$\pm$7.26\\ \hline
\end{tabular}
}
\caption{Comparison of HVAdam, AdaBelief, and Padam. A Lower FID ($[\mu \pm \sigma]$) is better.}
\label{table:padam_compare}
\end{table}

We also do the experiments on diffusion model, which are based on the official code in \footnote{https://github.com/openai/improved-diffusion}. The results are shown in Tabel~\ref{table:diffusion}.
\begin{table}[h]
\centering
\begin{tabular}{c|c|c|c}
\hline
    & \textbf{HVAdam} & Adam & SGD \\ \hline
FID    & \textbf{11.68} & 15.09 & 82.29\\ \hline
\end{tabular}
\caption{Comparison of HVAdam, Adam, and SGD. A Lower FID ($[\mu \pm \sigma]$) is better.}
\label{table:diffusion}
\end{table}

\end{document}